\documentclass{article}

\usepackage{microtype}
\usepackage{graphicx}
\usepackage{booktabs} 

\usepackage{hyperref}

\usepackage[accepted]{icml2021}

\icmltitlerunning{Robust Reinforcement Learning using Least Squares Policy Iteration with Provable Performance Guarantees}

\usepackage{textcomp}
\usepackage{xcolor}

\usepackage{amsmath,amsthm,amstext,amsfonts,amssymb,mathrsfs}
\usepackage{graphicx,caption,subcaption,algorithm,algorithmic} 

\usepackage{wrapfig,bbm,booktabs}

\newtheorem{lemma}{Lemma}
\newtheorem{assumption}{Assumption}

\newtheorem{proposition}{Proposition}
\newtheorem{theorem}{Theorem}

\newtheorem{remark}{Remark}

\newcommand{\dd}{\mathrm{d}}
\newcommand{\E}{\mathbb{E}}

\newcommand{\R}{\mathbb{R}}
\newcommand{\Ss}{\mathcal{S}}
\newcommand{\Aa}{\mathcal{A}}
\newcommand{\Pp}{\mathcal{P}}
\newcommand{\U}{\mathcal{U}}

\newcommand{\F}{\mathcal{F}}

\DeclareMathOperator*{\argmax}{arg\,max}

\begin{document}
	
	\twocolumn[
	\icmltitle{Robust Reinforcement Learning using Least Squares Policy Iteration \\ with Provable Performance Guarantees}
	
	
	
	\icmlsetsymbol{equal}{*}
	
	\begin{icmlauthorlist}
		\icmlauthor{Kishan Panaganti}{to}
		\icmlauthor{Dileep Kalathil}{to}
	\end{icmlauthorlist}
	
	\icmlaffiliation{to}{Department of Electrical and Computer Engineering, Texas A\&M University, College Station, United States}
	
	
	
	\vskip 0.3in
	]
	
	
	
	\printAffiliationsAndNotice{}  

\begin{abstract}
This paper addresses the problem of model-free reinforcement learning for Robust Markov Decision Process (RMDP) with large state spaces. The goal of the RMDP framework is to find a policy that is robust against the parameter uncertainties due to the  mismatch between the simulator model and  real-world settings.  We first propose the Robust Least Squares Policy Evaluation algorithm, which is a multi-step online model-free learning algorithm for policy evaluation. We prove the convergence of this algorithm using stochastic approximation techniques.  We then propose Robust Least Squares Policy Iteration (RLSPI) algorithm for learning the optimal robust policy. We also give a general weighted Euclidean norm  bound on the error (closeness to optimality) of the resulting policy. Finally, we demonstrate the performance of our RLSPI algorithm on some standard benchmark problems.
\end{abstract}

\section{Introduction}

Model-free Reinforcement Learning (RL) algorithms typically learn a policy by training on a simulator. In  the RL literature, it is nominally assumed that the testing environment is identical to the training environment (simulator model). However, in reality, the parameters of the simulator model can be different from the real-world setting. This can be due to the approximation errors incurred while modeling, due to the changes in the real-world parameters over time, and can even be due to possible adversarial  disturbances in the real-world.  For example, in many robotics applications,  the standard simulator parameter settings  (mass, friction,  wind conditions, sensor noise, action delays) can be different from that of the actual robot in the real-world.  This mismatch between   the training and testing environment  parameters  can significantly degrade the real-world performance of the model-free learning algorithms trained on a simulator model.


The RMDP framework \citep{iyengar2005robust, nilim2005robust}  addresses the \textit{planning} problem of computing the optimal policy  that is robust against parameter uncertainties that cause the mismatch between the training and testing environment  parameters. The RMDP problem has been analyzed extensively in the tabular case \citep{iyengar2005robust, nilim2005robust, wiesemann2013robust, xu2010distributionally, yu2015distributionally} and  under the linear function approximation \citep{tamar2014scaling}. Algorithms for learning the optimal robust policy  with provable guarantees  have  been proposed, both in the model-free  \citep{roy2017reinforcement} and model-based \citep{lim2013reinforcement}   reinforcement learning settings.   However, the theoretical guarantees from these works are limited to  the tabular RMDP settings.  Learning policies for problems  with  large state spaces is computationally  challenging. RL algorithms typically overcome this issue by using function approximation architectures, such as  linear basis functions \citep{lagoudakis2003least}, reproducing kernel Hilbert spaces (RKHS) \citep{yang2020reinforcement} and deep neural networks \citep{lillicrapHPHETS15}. Recently, robust reinforcement learning problem has been addressed  using deep RL methods  \citep{pinto2017robust, Mankowitz2020Robust, zhang2020robust, derman2018soft, vinitsky2020robust}. However, these works are  empirical in nature and do not provide any theoretical guarantees for the learned policies. The problem of learning optimal robust policies with provable performance guarantees for RMDPs with  large state spaces  has not been well studied in the literature.



 
In this paper, we address the problem of learning a  policy that is provably robust against the parameter uncertainties for RMDPs with  large state spaces. In particular, we propose an online model-free reinforcement learning algorithm with linear function approximation for learning the optimal robust policy, and provide theoretical guarantees on the performance of the learned policy.   Our choice of  linear function approximation is motivated by its analytical tractability while providing the scaling to large  state  spaces. Indeed, linear function approximation based approaches have been successful in providing algorithms with provable guarantees for many challenging problems in RL, including online model-free exploration \citep{jin2020provably,yang2020reinforcement}, imitation learning \citep{abbeel2004apprenticeship,arora2020provable}, meta reinforcement learning \citep{wang2020global,kong2020meta}, and offline reinforcement learning \citep{wang2021what,duan2020minimax}. Robust RL is much more challenging than the standard  (non-robust) RL problems due to the inherent  nonlinearity associated with the robust dynamic programming.  We overcome this issue by a cleverly designed approximate dynamic programming approach. We then propose a model-free robust policy iteration using this approach  with provable guarantees.  Our algorithmic and technical contributions are as follows:

(i) Robust Least Squares Policy Evaluation (RLSPE($\lambda$)) algorithm: A learning-based policy iteration algorithm  needs to learn  the value of a policy for performing (greedy) policy improvement. For this, we first propose RLSPE($\lambda$) algorithm, a multi-step, online, model-free policy evaluation algorithm with linear function approximation. This can be thought as the  robust version of classical least squares based RL algorithms for policy evaluation, like LSTD($\lambda$) and LSPE($\lambda$). We  prove the convergence of this algorithm using stochastic approximation techniques, and also  characterize its approximation error due to the linear architecture. 

(ii) Robust Least Squares Policy Iteration (RLSPI) algorithm: We  propose the RLSPI  algorithm for learning the optimal robust policy. We also give a general $L_{2}$-norm bound on the error (closeness to optimality) of the resulting policy at any iterate of the algorithm. To the best of our knowledge, this is the first work that presents a learning based policy iteration algorithm for robust reinforcement learning with such provable guarantees.  

(iii)  Finally, we demonstrate the performance of the RLSPI algorithm on various  standard RL test environments.



\subsection{Related Work}
RMDP formulation to address the parameter uncertainty problem was first proposed by \citep{iyengar2005robust} and \citep{nilim2005robust}. \citep{iyengar2005robust} showed that the optimal robust value function and policy can be computed using the robust counterparts of the standard value iteration and policy iteration.  To tackle the parameter uncertainty problem, other works considered distributionally robust setting \citep{xu2010distributionally}, modified policy iteration \citep{kaufman2013robust}, and {more general uncertainty set} \citep{wiesemann2013robust}. We note that the focus of these works were mainly on the planning problem in the tabular setting.  Linear function approximation method to solve large  RMDPs was proposed in \citep{tamar2014scaling}. Though this work suggests a sampling based approach, a general model-free learning algorithm and analysis was not included. \citep{roy2017reinforcement} proposed the robust versions of the classical model-free  reinforcement learning algorithms such as Q-learning, SARSA, and TD-learning in the tabular setting. They also proposed function approximation based algorithms for  the policy evaluation. However,  this work does not have a policy iteration algorithm with provable guarantees for learning the optimal robust policy.   \citep{derman2018soft}  introduced soft-robust actor-critic algorithms using neural networks, but does not provide any global convergence  guarantees for the learned policy. \citep{tessler2019action} proposed a min-max game framework to address the robust learning problem focusing on the tabular setting. \citep{lim2019kernel}  proposed a kernel-based RL algorithm for finding the robust value function in a batch learning setting.  \citep{Mankowitz2020Robust} employed an entropy-regularized policy optimization algorithm for continuous control using neural network, but does not provide any provable guarantees for the learned policy.




Our work differs from the above in two significant ways. Firstly, we  develop a new multi-step model-free reinforcement learning algorithm, RLSPE($\lambda$), for policy evaluation.  Extending the classical least squares based policy evaluation algorithms, like LSPE($\lambda$) and LSTD($\lambda$) \citep{BerBook12}, to the robust case is very challenging due to the nonlinearity of the robust TD($\lambda$) operator. We overcome this issue by a cleverly defined approximate  robust TD($\lambda$) operator that is amenable to online learning using least squares approaches. Also, as pointed out in \citep{bertsekas2011approximate}, convergence analysis of least squares style algorithms for RL is different from that of the standard temporal difference (TD) algorithm. Secondly, we develop a new robust policy iteration algorithm with provable guarantees on the performance of the policy at any iterate. In particular, we give a general weighted Euclidean norm bound on the error of the resulting policy. While similar results are available for the non-robust settings, this is the first work to provide such a characterization in the challenging setting of robust reinforcement learning.

\section{Background and Problem Formulation}
\label{sec:formulation}

A Markov Decision Process is a tuple $M = (\Ss, \Aa, r, P, \alpha)$ where $\Ss$ is the  state space, $\Aa$ is the  action space,  $r: \Ss\times \Aa\rightarrow \R$ is the reward function, and $\alpha \in (0, 1)$ is the discount factor. The transition probability matrix $P_{s,a}(s')$ represents the probability of transitioning to state $s'$ when  action $a$ is taken at state $s$. We consider a finite MDP setting where the cardinality of state and action spaces are finite (but very large).    A  policy $\pi$ maps each state to an action. The value of a policy $\pi$ evaluated at state $s$ is given by
\begin{align*}
V_{\pi, P}(s) = \E_{\pi, P}[\sum^{\infty}_{t=0} \alpha^{t} r(s_{t}, a_{t}) ~|~ s_{0} = s], \end{align*}where $a_{t} \sim \pi(s_{t})$ and $s_{t+1} \sim P_{s_t,a_t}(\cdot)$. The optimal  value function and the  optimal policy of an MDP with the transition probability $P$ is defined as $V^{*}_{P} = \max_{\pi} V_{\pi, P}$ and $\pi^{*}_{P} = \argmax_{\pi} V_{\pi, P}$.

The RMDP formulation considers a set of model parameters (uncertainty set) under the assumption that the actual parameters lie in this  uncertainty set, and the  algorithm computes a robust policy that performs  best under the worst model.  More precisely, instead of a fixed  transition probability matrix $P$, we consider a set of transition probability matrices $\mathcal{P}$.  We assume that the set $\mathcal{P}$ satisfies the standard \textit{rectangularity condition} \citep{iyengar2005robust}. The objective is to find a policy that maximizes the worst-case performance.  Formally, the \textit{robust value function} $V_{\pi}$ corresponding to a policy $\pi$ and the \textit{optimal robust value function} $V^{*}$ are defined as
\citep{iyengar2005robust,nilim2005robust}
\begin{align}
\label{eq:Vpi-Vstart-1}
V_{\pi} = \inf_{P \in \mathcal{P}} ~V_{\pi, P},~~ V^{*} = \sup_{\pi} \inf_{P \in \mathcal{P}} ~V_{\pi, P} . 
\end{align} 
The \textit{optimal robust policy} $\pi^{*}$ is such that the robust value function corresponding to it matches the optimal robust value function, that is, $V_{\pi^*} =V^{*} $.

A generic characterization of the set  $\mathcal{P}$ makes the RMDPs problems intractable to solve by model-free methods. In the standard model-free methods, the algorithm has access to a simulator that can simulate the next state given the current state and current action, according to a fixed transition probability matrix (that is unknown to the algorithm). However, generating samples according to each and every transition probability matrix from the set $\mathcal{P}$ is clearly infeasible. To overcome this difficulty, we use the characterization of the uncertainty set used in \citep{roy2017reinforcement}. 
\begin{assumption}[Uncertainty Set]
Each $P \in \mathcal{P}$ can be represented as $P_{s,a}(\cdot) = P^{o}_{s,a}(\cdot) + U_{s,a}(\cdot)$ for some $U_{s,a} \in \mathcal{U}_{s,a}$, where $P^{o}_{s,a}(\cdot)$ is the unknown transition probability matrix corresponding to the nominal (simulator) model and $\mathcal{U}_{s,a}$ is a confidence region around it. 
\end{assumption}
Using the above characterization, we can write $\mathcal{P} = \{P^{o} + U : U \in \mathcal{U} \}$, where $\mathcal{U} = \cup_{s,a} \mathcal{U}_{s,a}$. So,  $\mathcal{U}$ is the set of all possible perturbations to the nominal model $P^{o}$.

An example of the uncertainty set $\U$ can be the spherical uncertainty set with a radius parameter.  Define $\U_{s,a}:= \{ x~|~\|x\|_2 \leq r, \sum_{s\in\Ss} x_s = 0, -P^o_{s,a}(s') \leq x_{s'} \leq 1-P^o_{s,a}(s'), \forall s' \in \Ss \},$ for all $(s,a)\in(\Ss,\Aa)$, for some  $r>0$. Notice that, this uncertainty set uses the knowledge of the nominal model $P^o$ in its construction. In practice, we do not know $P^o$. So,  in Section \ref{sec:rlspe-1}, we introduce an approximate uncertainty set without using this information. 

We consider \textit{robust Bellman operator for policy evaluation}, defined as \citep{iyengar2005robust}
\begin{align}
\label{eq:Tpi-basic}
\hspace*{-0.2cm} T_{\pi} (V) (s) =  r(s,\pi(s)) + \alpha \inf_{P \in \mathcal{P}}   \sum_{s'} P_{s,\pi(s)}(s') V(s'),\hspace*{-0.2cm}
\end{align} a popular approach to solve \eqref{eq:Vpi-Vstart-1}. Using our characterization of the uncertainty set, we can rewrite \eqref{eq:Tpi-basic} as
\begin{align}
\label{eq:Tpi-1}
T_{\pi} (V) (s) &=  r(s,\pi(s)) + \alpha  \sum_{s'} P^{o}_{s,\pi(s)}(s') V(s') \nonumber \\&\hspace{0.5cm}+ \alpha \inf_{U \in \mathcal{U}_{s,\pi(s)}}   \sum_{s'} U_{s,\pi(s)}(s') V(s').
\end{align}
For any set  $\mathcal{B}$ and a vector $v$, define $
\sigma_{\mathcal{B}} (v) = \inf \{ u^{\top}v : u \in \mathcal{B} \}.$ We denote $|\Ss|$ as the cardinality of the set $\Ss$. Let ${\sigma}_{\mathcal{U}_\pi} (v)$ and  $r_{\pi}$ be the $|\mathcal{S}|$ dimensional column vectors defined as $ (\sigma_{\mathcal{U}_{s,\pi(s)}}(v) : s \in \mathcal{S})^{\top}$ and   $(r(s,\pi(s)): s \in \mathcal{S})^\top$, respectively. Let  $P^{o}_{\pi}$ be the stochastic matrix corresponding to the policy $\pi$ where for any $s,s'\in\Ss,$ $P^{o}_{\pi}(s,s') = P^{o}_{s,\pi(s)}(s')$. Then, \eqref{eq:Tpi-1} can be written in the matrix form as
\begin{align}
\label{eq:Tpi-matrix-1}
T_{\pi} (V)  =  r_{\pi} + \alpha  P^{o}_{\pi} V + \alpha \sigma_{U_{\pi}} (V).
\end{align}
It is known  \citep{iyengar2005robust} that $T_{\pi}$ is a contraction in sup norm and the robust value function $V_{\pi}$ is the unique fixed point of $T_{\pi}$. The \textit{robust Bellman operator} $T $ can also be defined in the same way as in the non-robust setting,
\begin{align}
\label{eq:T}
T (V)  =  \max_{\pi}T_{\pi} (V).
\end{align}
It is also known   \citep{iyengar2005robust} that $T$ is a contraction in sup norm, and the optimal robust value function $V^{*}$   is its unique fixed point. 

The goal of the robust RL is to learn the optimal robust policy $\pi^{*}$ without knowing the nominal model $P^{o}$ or the uncertainty set $\mathcal{P}$.


\section{Robust Least Squares Policy Evaluation}
\label{sec:RPE}  
In this section, we develop the RLSPE($\lambda$) algorithm for learning the robust value function.

\subsection{Robust TD$(\lambda)$ Operator and the Challenges}  
\label{sec:RTB-Ber}
In RL,  a very useful approach for analyzing the multi-step learning algorithms  like TD($\lambda$), LSTD($\lambda$), and LSPE($\lambda$) is to define a multi-step Bellman operator called  TD($\lambda$) operator \citep{tsitsiklis1997analysis} \citep{BerBook12}.  Following the same approach, we can define  the robust TD($\lambda$) operator as well. For a given policy $\pi$, and a parameter $\lambda \in [0, 1),$ the robust TD($\lambda$) operator denoted by ${T}^{(\lambda)}_\pi : \mathbb{R}^{|\mathcal{S}|} \rightarrow \mathbb{R}^{|\mathcal{S}|}$ is defined as
\begin{align}
\label{eq:TLambda-Ber-1}
{T}^{(\lambda)}_\pi (V )= (1-\lambda) \sum_{m=0}^\infty \lambda^{m} T^{m+1}_{\pi}(V). 
\end{align} 
Note that for $\lambda = 0,$ we recover $T_{\pi}$. The following result is straightforward.
\begin{proposition}[informal]
\label{prop:informal-cont-fp-1}
${T}^{(\lambda)}_\pi$ is a contraction in  sup norm  and  the robust value function $V_{\pi}$ is its unique fixed point, for any $\alpha \in (0, 1), \lambda \in [0, 1)$.
\end{proposition}



For RMDPs with very large state space, exact dynamic  programming methods which involve  the evaluation of \eqref{eq:Tpi-1} or \eqref{eq:TLambda-Ber-1} are intractable.  A standard approach to overcome this issue is to approximate the value function using some function approximation architecture. Here we focus on linear function approximation architectures \citep{BerBook12}. In linear function approximation architectures, the value function is represented as the weighted sum of features as, $\bar{V}(s) = \phi(s)^{\top} w, \forall s \in \mathcal{S},$ where $ \phi(s)  = (\phi_{1}(s), \phi_{2}(s), \ldots, \phi_{L}(s))^{\top}$ is an $L$ dimensional feature vector with $L < |\mathcal{S}|$, and  $w = (w_1,\cdots,w_{L})^\top$ is a weight vector. In the matrix form, this can be written as $\bar{V} = \Phi w $ where  $\Phi$ is an $|\mathcal{S}| \times L$  dimensional feature matrix whose $s^{\text{th}}$ row  is $\phi(s)^{\top} $. We assume linearly independent columns for $\Phi$, i.e., rank($\Phi$) = $L$. 

The standard approach to find an approximate (robust) value function is  to solve for  a $w_{\pi}$, with $\bar{V}_{\pi} =  \Phi
 w_{\pi}$, such that $\Phi w_{\pi} = \Pi T^{(\lambda)}_{\pi} \Phi w_{\pi}$, where $\Pi$ is a   projection  onto the subspace spanned by the columns of $\Phi$. The projection is with  respect to a $d$-weighted  Euclidean norm. This norm is defined as  $\| V \|^{2}_{d} =  V^{\top} D V$, where $D$ is a diagonal matrix with non-negative diagonal entries $(d(s), s \in \mathcal{S})$, for any  vector $V$. Under suitable assumptions, \citep{tamar2014scaling} showed that $\Pi T_{\pi}$ is a contraction in a  $d$-weighted  Euclidean norm. We also use a similar assumption stated below.
\begin{assumption}
\label{as:offpolicy}
(i) For any given policy $\pi$, there exists an exploration policy $\pi_{e} =  \pi_{\text{exp}}(\pi)$ and a $\beta \in (0,1)$ such that $\alpha P_{s,\pi(s)}(s') \leq \beta P^{o}_{s,\pi_{e}(s)}(s')$, for all transition probability matrices $P\in\Pp$ and for all states $s,s'\in\Ss$.  \\
(ii) There exists a steady state distribution $d_{\pi_{e}} = (d_{\pi_{e}}(s), s \in \mathcal{S})$ for the Markov chain with transition probability $P^{o}_{\pi_{e}}$ with $d_{\pi_{e}}(s) > 0, \forall s \in \mathcal{S}$.
\end{assumption}
In the following, we will simply use $d$ instead of $d_{\pi_{e}}$. 


Though the above assumption appears restrictive, it is necessary to show that  $\Pi T_{\pi}$  is a contraction in the $d$-weighted Euclidean norm, as proved in \citep{tamar2014scaling}. Also, a similar assumption is used in proving the convergence of off-policy reinforcement learning algorithm  \citep{bertsekas2009projected}. In the robust case, we can expect a similar condition because  we are learning a robust value function for a set of transition probability matrices instead of a single transition probability matrix. We can now show the following. 

\begin{proposition}[informal]
\label{prop:informal-cont-fp-2}
Under Assumption \ref{as:offpolicy}, $\Pi {T}^{(\lambda)}_\pi$ is a contraction mapping in the $d$-weighted Euclidean norm for any $\lambda \in [0, 1)$. 
\end{proposition}


The linear approximation based robust value function $\bar{V}_{\pi} =\Phi w_{\pi}$ can be computed using the iteration, $\Phi w_{k+1} = \Pi T^{(\lambda)}_{\pi} \Phi w_{k}$. Since  $\Pi T^{(\lambda)}_{\pi}$ is a contraction, $w_{k}$ will converge to $w^{*}$.  A closed form solution for $w_{k+1}$ given $w_{k}$ can be found by {\bf least squares approach} as $
w_{k+1} = \arg \min_{w} \|\Phi w - \Pi T^{(\lambda)}_{\pi} \Phi w_{k}\|^{2}_{d}$. It can be shown that (details are given in the supplementary material), we can get a closed form solution for $w_{k+1}$  as \begin{align}
\label{eq:rpvi-11} 
w_{k+1} = w_{k} + (\Phi^{\top} D \Phi)^{-1} \Phi^{\top} D (T^{(\lambda)}_{\pi} \Phi w_{k}  - \Phi w_{k}). 
\end{align}This is similar to the  projected equation approach \citep{BerBook12} in the non-robust setting. Even in the non-robust setting,  iterations using the \eqref{eq:rpvi-11} is intractable for MDPs with large state space. Moreover, when the transition  matrix is unknown, it is not feasible to use \eqref{eq:rpvi-11} exactly  even for small RMDPs.  Simulation-based model-free learning algorithms are developed for addressing this problem in the  non-robust case.  In particular, LSPE($\lambda$) algorithm   \citep{nedic2003least, BerBook12}  is used to solve the iterations of the above form.

However, compared to the non-robust setting, there are two significant challenges in learning the robust value function  by  using simulation-based model-free approaches.

(i) Non-linearity of the robust TD($\lambda$) operator:  The non-robust $T_{\pi}$ operator and the TD($\lambda$)  operator do not involve any nonlinear operations. So, they can be estimated efficiently from simulation samples in  a model-free way.  However, the robust TD($\lambda$) operator when expanded will have the  following form (derivation is given in the supplementary material).
\begin{align}
\label{eq:TLambda-Ber-expanded}
T^{(\lambda)}_{\pi}(V) &= (1 - \lambda) \sum^{\infty}_{m=0} \lambda^{m} (\sum^{m}_{k=0} (\alpha P^{o}_{\pi})^{k} r_{\pi} + (\alpha P^{o}_{\pi})^{m+1} V \nonumber \\&\hspace{1cm}+ \alpha \sum^{m}_{k=0} (\alpha P^{o}_{\pi})^{k} \sigma_{\mathcal{U}_{\pi}} (T^{(m-k)}_\pi V)). 
\end{align}
The last term is very difficult to estimate using simulation-based model-free approaches due to the composition of operations $\sigma_{\U_{\pi}}$ and $T_\pi$. In addition, nonlinearity of the $T_{\pi}$ operator by itself adds to the complexity. 

(ii) Unknown uncertainty region $\mathcal{U}$: In our formulation, we assumed that the transition probability uncertainty  set $\mathcal{P}$ is given by $\mathcal{P} = P^{o} + \mathcal{U}$. So, for each $U \in \mathcal{U}$, $ P^{o} + U$ should be a valid transition probability matrix. However, in the model-free setting, we do not know the nominal transition probability $P^{o}$. So, it is not possible to know $ \mathcal{U}$ exactly a priori. One can only use an approximation  $	\widehat{\U}$ instead of $\mathcal{U}$. This can possibly affect the convergence of the learning algorithms.

\subsection{Robust Least Squares Policy Evaluation (RLSPE($\lambda$)) Algorithm } 
\label{sec:rlspe-1}
We overcome  the challenges of learning the robust value function by defining an   approximate robust TD($\lambda$) operator, and by developing a robust least squares policy evaluation algorithm based on that. 

Let $\widehat{\mathcal{U}}$ be the approximate uncertainty set we use instead of the actual uncertainty set. An example of the approximate uncertainty set $\widehat{\mathcal{U}}$ can be the spherical uncertainty set defined \textit{without} using the knowledge of the model $P^{o}$ as   $\widehat{\U}_{s,a}:= \{ x~|~\|x\|_2 \leq r, \sum_{s\in\Ss} x_s = 0\}$ for all $(s,a)\in(\Ss,\Aa)$. Note that $P^{o} + U$ for $U \in \widehat{\mathcal{U}}$ need not be a valid transition probability matrix and this poses challenges both for the algorithm and analysis.

For a given policy $\pi$ and a parameter $\lambda \in [0, 1),$ approximate robust TD($\lambda$) operator denoted by $\widetilde{T}^{(\lambda)}_\pi : \mathbb{R}^{|\mathcal{S}|} \rightarrow \mathbb{R}^{|\mathcal{S}|},$ is defined as
\begin{align}
\label{eq:TLambda-1}
\widetilde{T}^{(\lambda)}_\pi (V )&= (1-\lambda) \sum_{m=0}^\infty \lambda^{m} \left[\sum_{t=0}^{m} (\alpha P^{o}_\pi)^t {r}_\pi \right. \nonumber \\&\hspace{0.1cm} \left. + \alpha \sum_{t=0}^{m} (\alpha P^{o}_\pi)^t  ~ {\sigma}_{\widehat{\U}_\pi} (V) + (\alpha P^{o}_\pi)^{m+1} V \right].
\end{align}
Note that even with $\widehat{\U}_\pi = {\U}_\pi$,  \eqref{eq:TLambda-1} is different from \eqref{eq:TLambda-Ber-expanded}. We will  show that this clever approximation  helps to overcome the  challenges due to the nonliterary associated with \eqref{eq:TLambda-Ber-expanded}.

However, we emphasize that \eqref{eq:TLambda-1} is not an arbitrary definition. Note that, for $\widehat{\U}_\pi = {\U}_\pi$, with $\lambda = 0,$ we recover the operator $T_{\pi}$. Moreover, the robust value function $V_{\pi}$ is a fixed point of $\widetilde{T}^{(\lambda)}_\pi $ when  $\widehat{\U}_\pi = {\U}_\pi$ for any $\lambda \in [0, 1).$ We state this formally below.


\begin{proposition}
\label{prop:fixedpoint-1}
Suppose $\widehat{\U}_\pi = {\U}_\pi$. Then, for any $\alpha \in(0, 1)$ and  $\lambda \in [0, 1),$ the robust value function $V_{\pi}$ is a fixed point of $\widetilde{T}^{(\lambda)}_\pi$, i.e., $\widetilde{T}^{(\lambda)}_\pi  (V_\pi) = V_\pi.$ 
\end{proposition} 

Intuitively, the convergence of any learning algorithm using the approximate robust TD($\lambda$) operator will depend on the difference between the   actual uncertainty set ${\U}$ and its approximation $\widehat{\U}$. To quantify this, we use the following metric.  Let $\rho=\max_{s\in\Ss, a\in\Aa} \rho_{s,a}$ where 
\begin{align*}
    \rho_{s,a} = \max \hspace{-0.1cm}\left\{\hspace{-0.2cm}\begin{array}{ll} \max_{x\in \widehat{\U}_{s,a}}\max_{y\in {\U}_{s,a} \setminus \widehat{\U}_{s,a}} \| x-y\|_d / d_{\text{min}},\\ \max_{x\in {\U}_{s,a}}\max_{y\in \widehat{\U}_{s,a} \setminus {\U}_{s,a}} \| x-y\|_d / d_{\text{min}} \end{array} \hspace{-0.2cm}\right\} 
\end{align*}
and $d_{\text{min}} := \min_{s \in\Ss} d(s).$ By convention, we set $\rho_{s,a}=0$  when $\widehat{\U}_{s,a} = {\U}_{s,a}$ for all $(s,a)\in(\Ss,\Aa).$ So, $\rho = 0$ if $\widehat{\U} = {\U}$.  Using this characterization and under some additional assumptions on the discount factor, we show that the approximate robust TD($\lambda$) operator  is a contraction in the  $d$-weighted Euclidean norm.
\begin{theorem}
\label{prop:approx-td-op-1}
Under Assumption \ref{as:offpolicy}, for any $V_{1}, V_{2} \in \mathbb{R}^{|\mathcal{S}|}$ and $\lambda \in [0, 1),$ 
\begin{align}
\label{eq:Ttilde-contraction-1} \hspace{-0.3cm}
\|\Pi \widetilde{T}^{(\lambda)}_\pi V_{1} -  \Pi \widetilde{T}^{(\lambda)}_\pi V_{2} \|_{d} \leq  c(\alpha, \beta, \rho, \lambda) ~ \|V_{1} - V_{2}\|_{d}, \hspace{-0.2cm}
\end{align} 
where $c(\alpha, \beta, \rho, \lambda) = {(\beta(2-\lambda) + \rho\alpha  )}/{(1 - \beta  \lambda)}$. So, if $c(\alpha, \beta, \rho, \lambda)  < 1$, $\Pi \widetilde{T}^{(\lambda)}_\pi$ is a contraction in the  $d$-weighted Euclidean norm. Moreover, there exists a unique $w_{\pi}$ such that $\Phi w_{\pi} = \Pi \widetilde{T}^{(\lambda)}_\pi (\Phi w_{\pi})$. Furthermore, for this $w_{\pi}$,
\begin{align}
\label{eq:Ttilde-contraction-3}
&\| V_\pi - \Phi w_\pi \|_d \leq \nonumber \\&\hspace{0.2cm}\frac{1}{1 - c(\alpha, \beta, \rho, \lambda)} \left(\| V_\pi -  \Pi V_\pi  \|_d +  \frac{\beta \rho \| V_{\pi}\|_d}{1-\beta\lambda} \right). 
\end{align}
\end{theorem} 


%

We note that despite the  assumption on the discount factor, we  empirically show in Section \ref{sec:simulations} that our learning algorithm converges to a robust policy even if this assumption is violated. We also note that the upper bound in \eqref{eq:Ttilde-contraction-3} quantifies the \emph{error} of approximating the robust value function $V_\pi$ with the approximate robust value function $\Phi w_\pi$. We will later use this error bound in in characterizing performance of both RLSPE and RLSPI algorithms. 

Using the contraction property of approximate robust TD($\lambda$) operator, the linear approximation based robust value function $\bar{V}_{\pi} =\Phi w_{\pi}$ can be computed using the iteration, $\Phi w_{k+1} = \Pi \widetilde{T}^{(\lambda)}_{\pi} \Phi w_{k}$. Similar to \eqref{eq:rpvi-11},  we can get a closed form solution for $w_{k+1}$  using \textbf{least squares approach}  as
\begin{align}
\label{eq:rpvi-approx-1}
\hspace{-0.2cm} w_{k+1} = w_{k} + (\Phi^{\top} D \Phi)^{-1} \Phi^{\top} D (\widetilde{T}^{(\lambda)}_{\pi} \Phi w_{k}  - \Phi w_{k}).
\end{align}
This can be written in a more succinct matrix form as given below (derivation is given in the supplementary material).
\begin{align}
\label{eq:rpvi-approx-matrix-1}
w_{k+1} &= w_{k} + B^{-1}(A w_{k} +  C(w_{k}) + b),~~ \text{where}, \\
\label{eq:AB-1}
&A =  \Phi^{\top} D  (\alpha P^{o}_{\pi} - I) \sum^{\infty}_{m=0}  (\alpha \lambda P^{o}_{\pi})^{m}  \Phi,  \\
&B =   \Phi^{\top} D \Phi, \\ \label{eq:Cb-1}
&C(w) = \alpha \Phi^{\top} D  \sum^{\infty}_{t=0} (\alpha \lambda P^{o}_{\pi})^{t}  {\sigma}_{\widehat{\U}_{\pi}} (\Phi w), \\ &  b = \Phi^{\top} D  \sum^{\infty}_{t=0} (\alpha \lambda P^{o}_{\pi})^{t} r_{\pi}.
\end{align}
Iterations by evaluating \eqref{eq:rpvi-approx-matrix-1} exactly is intractable for MDPs with large state space, and infeasible if we do not know the transition probability $P^{o}_\pi$. To address this issue, we propose a simulation-based model-free online reinforcement learning algorithm, which we call robust least squares policy evaluation (RLSPE($\lambda$)) algorithm, for learning the robust value function. 

{\bf RLSPE($\lambda$) algorithm:} Generate a sequence of states and rewards, $(s_{t}, r_{t}, t \geq 0),$ using  the policy $\pi$. Update the parameters as
\begin{align}
\label{eq:RLSPE-1}
&\hspace{-0.2cm}w_{t+1} = w_{t} + \gamma_{t} B^{-1}_{t}(A_{t} w_{t} + b_{t} + C_{t}(w_{t})),~ \text{where},\\
\label{eq:AB-2}
&A_t = \frac{1}{t+1} \sum_{\tau=0}^{t} z_{\tau} ~ (\alpha \phi^{\top}(s_{\tau + 1}) - \phi{^\top}(s_\tau)), \\ &B_t = \frac{1}{t+1} \sum_{\tau=0}^{t} \phi(s_\tau) \phi{^\top}(s_\tau),\\
&C_t(w) = \frac{\alpha}{t+1}  \sum_{\tau=0}^{t} z_{\tau}  ~ \sigma_{\widehat{\U}_{s_{\tau}, \pi(s_{\tau})}} ( \Phi w), \\&b_t = \frac{1}{t+1}  \sum_{\tau=0}^{t}  z_{\tau}  r(s_{\tau},\pi(s_{\tau})),\\ &z_{\tau} = \sum_{m=0}^{\tau} (\alpha \lambda)^{\tau-m} \phi(s_{m}), \label{eq:z-2}
\end{align}
where $\gamma_{t}$ is a deterministic sequence of step sizes. We assume that the step size satisfies the  the standard Robbins-Munro stochastic conditions for stochastic approximation, i.e., $\sum^{\infty}_{t=0} \gamma_{t} = \infty,~ \sum^{\infty}_{t=0} \gamma^{2}_{t} < \infty$. 

We use the on-policy version of the RLSPE($\lambda$) algorithm in the above description. So,  we implicitly assume that the given policy $\pi$ is an exploration policy according to the Assumption \ref{as:offpolicy}. This is mainly for the clarity of the presentation and notational convenience. Also, this simplifies the presentation of the policy iteration algorithm introduced in the next section.  An  off-policy version of the above algorithm can be implemented using the techniques given in \citep{bertsekas2009projected}. We now give the convergence result of the RLSPE($\lambda$) algorithm. 
\begin{theorem}
\label{thm:RLSPE-convergence-1}
Let Assumption \ref{as:offpolicy} hold. Also, let $c(\alpha, \beta, \rho, \lambda) < 1 $ so that $\Pi \widetilde{T}^{(\lambda)}_{\pi}$ is a contraction according to Theorem \ref{prop:approx-td-op-1}.  Let $\{w_{t}\}$ be the sequence generated by the RLSPE($\lambda$) algorithm given in \eqref{eq:RLSPE-1}. Then,  $w_{t}$ converges to $w_{\pi}$  with probability 1 where $w_{\pi}$ satisfies the fixed point equation $\Phi w_{\pi}= \Pi \widetilde{T}^{(\lambda)}_{\pi} \Phi w_{\pi}$.
\end{theorem}

The key idea of the proof is to show that the RLSPE($\lambda$) update \eqref{eq:RLSPE-1} approximates the exact update equation \eqref{eq:rpvi-approx-1} and both converge to the same value $w_{\pi}$. One particularly challenging task is in analyzing the behavior of the term $C_{t}(w)$ due to the non-linearity of the function ${\sigma}_{\widehat{\U}_\pi} (.)$. We use the tools from stochastic approximation theory \citep{borkar2009stochastic, nedic2003least} to show this rigorously after establishing the tractable properties of the function ${\sigma}_{\widehat{\U}_\pi} (.)$. 

Note that Theorem \ref{thm:RLSPE-convergence-1} and  Theorem \ref{prop:approx-td-op-1} together give an error bound for the converged solution of the RLSPE($\lambda$) algorithm. More precisely,  Theorem \ref{thm:RLSPE-convergence-1} shows the convergence of the  RLSPE($\lambda$) algorithm to $w_{\pi}$ and Theorem \ref{prop:approx-td-op-1} gives the bound on $\| V_\pi - \Phi w_\pi \|_d$, which is the error due to linear function approximation. We will use this bound in the the convergence analysis of  the RLSPI algoirthm presented in the next section.


\section{Robust Least Squares Policy Iteration}
\label{sec:PI}

In this section, we introduce the robust least squares policy iteration (RLSPI) algorithm for finding the optimal robust policy.  RLSPI algorithm can be thought as the robust version of the  LSPI algorithm \citep{lagoudakis2003least}. RLSPI algorithm uses the RLSPE($\lambda$) algorithm for policy evaluation. However, model-free policy improvement is difficult when working with value functions since the policy update step will require us to solve \begin{align}
\label{eq:policy_update_value}
\pi_{k+1}  =  \argmax_{\pi}\widetilde{T}^{(\lambda)}_{\pi} (\bar{V}_k),
\end{align} where $\bar{V}_k$ is the approximate robust value function corresponding to the policy $\pi_k$, in the $(k+1)^{\text{th}}$ policy iteration loop. To overcome this, we first introduce the robust state-action value function (Q-function). 

For any given policy $\pi$ and  state-action pair $(s,a)$, we define the robust $Q$-value as,
\begin{align}
    \hspace{-0.3cm}Q_{\pi}(s,a) = \inf_{P \in \mathcal{P}}  \mathbb{E}_{P}[\sum^{\infty}_{t=0} \alpha^{t} r(s_{t}, a_{t}) ~| s_{0} = s, a_{0} =a].\hspace{-0.2cm}
\end{align}
Instead of learning the approximate robust value function $\bar{V}_{\pi}$, we can learn the approximate robust Q-value function $\bar{Q}_{\pi}$  using RLSPE($\lambda$). This can be done by defining the feature vector $\phi(s,a)$ where  $\phi(s,a) = (\phi_{1}(s,a), \ldots, \phi_{L}(s,a))^{\top}$ and the linear approximation of the form $\bar{Q}_{\pi}(s, a) = w^{\top} \phi(s,a)$ where  $w$ is a weight vector. The results from the previous section on the convergence of the RLSPE($\lambda$) algorithm applies for the case of learning Q-value function as well. 

RLSPI is a policy iteration algorithm that uses RLSPE($\lambda$) for policy evaluation at each iteration. It starts with an arbitrary initial policy $\pi_{0}$. At the $k$th iteration, RLSPE($\lambda$) returns a weight vector that represents the approximate Q-value function $\bar{Q}_{\pi_{k}} = \Phi w_{\pi_{k}} $ corresponding to the policy $\pi_{k}$. The next policy $\pi_{k+1}$ is the greedy policy corresponding to $\bar{Q}_{\pi_{k}}$, defined as $\pi_{k+1}(s) = \argmax_{a \in \mathcal{A}} \bar{Q}_{\pi_{k}}(s,a)$. For empirical evaluation purposes, we terminate the policy iteration for some finite value $K.$ RLSPI algorithm is summarized in Algorithm \ref{algo-lspI}.

\begin{algorithm}
	\caption{RLSPI Algorithm} 
	\label{algo-lspI}
	\begin{algorithmic}[1]
		\STATE Initialization: Policy evaluation weights error $\epsilon_0$, initial policy $\pi_0$. 
		\FOR {$k=0 \ldots K$}
		\STATE Initialize the policy weight vector $w_{0}$. Initialize time step $t \leftarrow 0.$
		    \REPEAT
		        \STATE Observe the state $s_{t},$ take action $a_{t} = \pi_{k}(s_{t}),$ observe  reward $r_{t}$ and  next state $s_{t+1}$.
		        \STATE Update the weight vector $w_{t}$ according to RLSPE($\lambda$) algorithm  (c.f. \eqref{eq:RLSPE-1}-\eqref{eq:z-2})
		        \STATE $t \leftarrow t+1$
		    \UNTIL{$\|w_{t}-w_{t-1}\|_2 < \epsilon_0$}
		    \STATE $w_{\pi_{k}} \leftarrow w_{t}$
		    \STATE Update the policy $$\pi_{k+1}(s) = \argmax_{a \in \mathcal{A}} \phi(s,a)^{\top} w_{\pi_{k}}$$
		\ENDFOR
	\end{algorithmic}
\end{algorithm}

We make the following assumptions for the convergence analysis of the RLSPI algoirthm. We note that we work with value functions instead of Q-value functions for notational convenience and consistency.

\begin{assumption}
\label{as:policy_iteration_req} (i) Each policy $\pi_{k}$ is an exploration policy, i.e. $\pi_{\text{exp}}(\pi_{k}) = \pi_{k}$.\\
(ii) The Markov chain $P^{o}_{\pi_{k}}$ has a stationary distribution $d_{\pi_{k}}$ such that $d_{\pi_{k}}(s) > 0, \forall s \in \mathcal{S}.$ \\
(iii) There exists a finite scalar $\delta$ such that ~$ \|V_{\pi_{k}} - \Pi_{d_{\pi_{k}}} V_{\pi_{k}} \|_{d_{\pi_{k}}}  < \delta$ for all $k$, where $\Pi_{d_{\pi_{k}}}$ is a projection onto the subspace spanned by the columns of $\Phi$ under the $d_{\pi_{k}}$-weighted  Euclidean norm. \\ 
(iv) For any probability distribution $\mu$, define probability another distribution $\mu_k = \mu H_k $ where $H_k$ is a stochastic matrix defined with respect to $\pi_{k}$. Also assume that there exists a probability distribution $\bar{\mu}$ and finite positive scalars $C_{1}, C_{2}$ such that $\mu_{k} \leq C_{1} \bar{\mu}$ and $d_{\pi_{k}} \geq \bar{\mu}/C_{2}$ for all $k$. 
\end{assumption}

We  note that these are  the standard assumptions used in the RL literature to provide theoretical guarantees for approximate policy/value iteration algorithms with linear function approximation in the non-robust settings \citep{munos2003error,munos08a, lazaric2012finite}.  We make no additional assumptions even though we are addressing the  more difficult robust RL problem. The specific form of the stochastic matrix $H_k$ specified in Assumption \ref{as:policy_iteration_req}.$(iv)$ is deferred to the proof of Theorem \ref{thm:RLSPI} for brevity of the presentation.

We now give the asymptotic convergence result for the RLSPI algorithm. We assume that, similar to the non-robust setting \citep{munos2003error}, the policy evaluation step (inner loop) is run to the convergence. 
We only present the case where $\rho = 0$. The proof for the general case is straightforward, but involves much more detailed algebra. So, we omit those details for the clarity of presentation. 
\begin{theorem}
\label{thm:RLSPI}
Let Assumption \ref{as:offpolicy} and Assumption \ref{as:policy_iteration_req} hold. Let $\{\pi_{k}\}$  be the sequence of the policies generated by the RLSPI algorithm. Let $V_{\pi_{k}}$ and $\bar{V}_{k} = \Phi w_{\pi_{k}}$ be true robust value function and the approximate robust value function corresponding to the policy $\pi_{k}$. Also, let  $V^*$ be the optimal robust value function. Then, with $c(\alpha,\beta,0,\lambda)<1$,
\begin{align} \label{eq:pi_thm_1}
    &\hspace{-0.2cm}\limsup_{k\to\infty} \| V^* - V_{\pi_k} \|_\mu\nonumber\\&\hspace{-0.1cm}\leq \frac{2\sqrt{C_1 C_2}~c(\alpha,\beta,0,\lambda) }{(1-c(\alpha,\beta,0,\lambda) )^2}~ \limsup_{k\to\infty} \|V_{\pi_{k}} - \bar{V}_{k}\|_{d_{\pi_k}}.
\end{align}
Moreover, from Theorem \ref{prop:approx-td-op-1} and  Assumption \ref{as:policy_iteration_req}.$(iii)$,   we have
\begin{align} \label{eq:pi_thm_2}
    \limsup_{k\to\infty} \| V^* - V_{\pi_k} \|_\mu\leq  \frac{2\sqrt{C_1 C_2}~c(\alpha,\beta,0,\lambda) }{(1-c(\alpha,\beta,0,\lambda) )^3}~ \delta  .
\end{align}
\end{theorem}

The above theorem, in particular \eqref{eq:pi_thm_2}, gives a (worst case) guarantee for the performance of the policy learned using the RLSPI algorithm. Note that the upper bound in \eqref{eq:pi_thm_2} is a constant where $\delta$ represents the (unavoidable) error due to the linear function approximation. We also note that using `$\limsup$' is necessary due to the policy chattering phenomenon in approximate policy iteration algorithms which exists even in the non-robust case \citep{BerBook12}. 

Instead of the asymptotic bound given in Theorem \ref{thm:RLSPI}, we can actually get a bound for any $K$ given in the RLSPI algorithm by modifying Assumption \ref{as:policy_iteration_req}.$(iv)$. We defer the precise statements of the assumption and theorem to  Section \ref{sec:PI_appendix} in the supplementary material due to page limitation.

\section{Experiments} \label{sec:simulations}

\begin{figure*}[t]
	\centering
	\begin{minipage}{.32\textwidth}
		\centering
		\includegraphics[width=\linewidth]{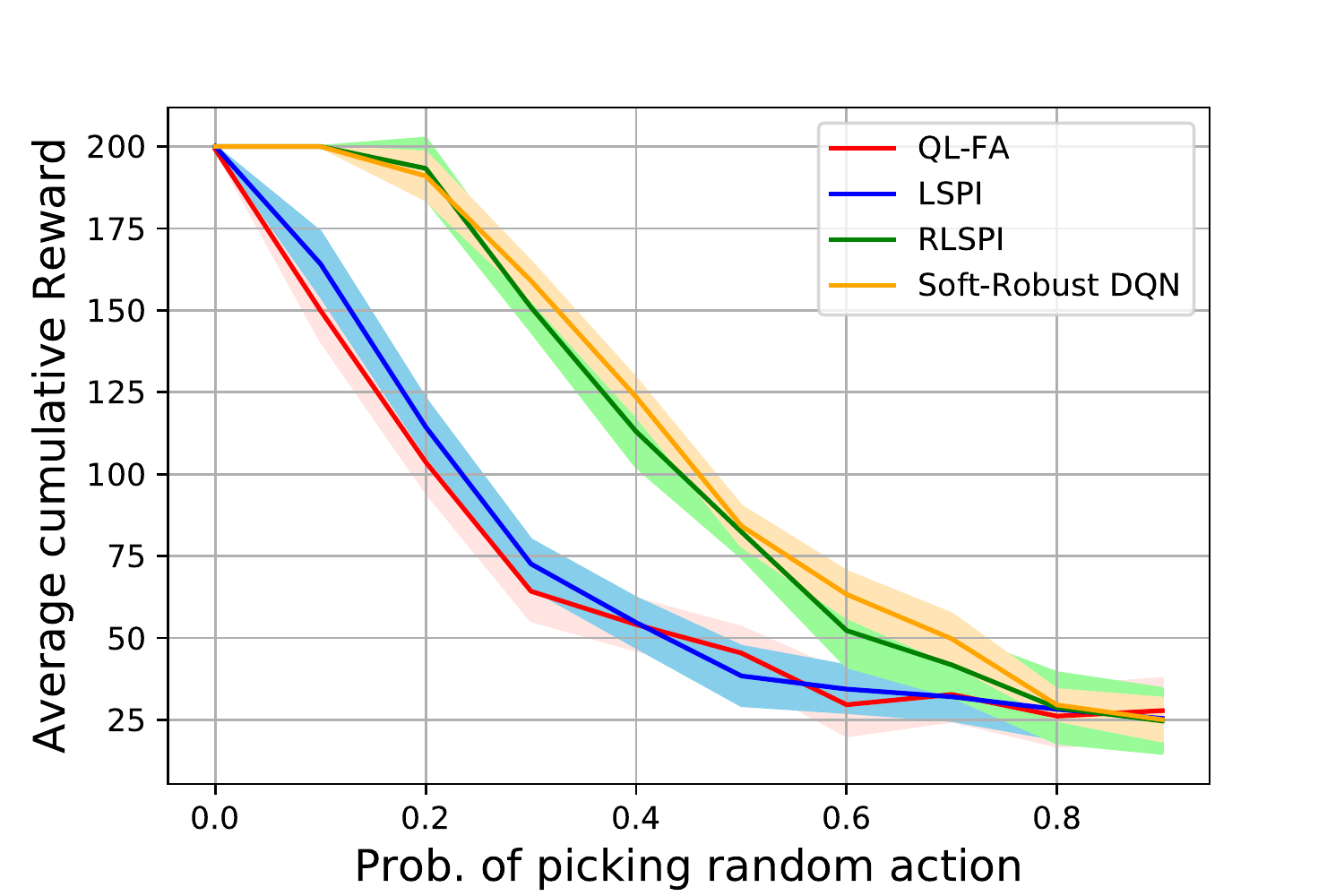}
		\captionof{figure}{CartPole}
		\label{fig:cart_failprob}
	\end{minipage}
	\begin{minipage}{.32\textwidth}
		\centering
		\includegraphics[width=\linewidth]{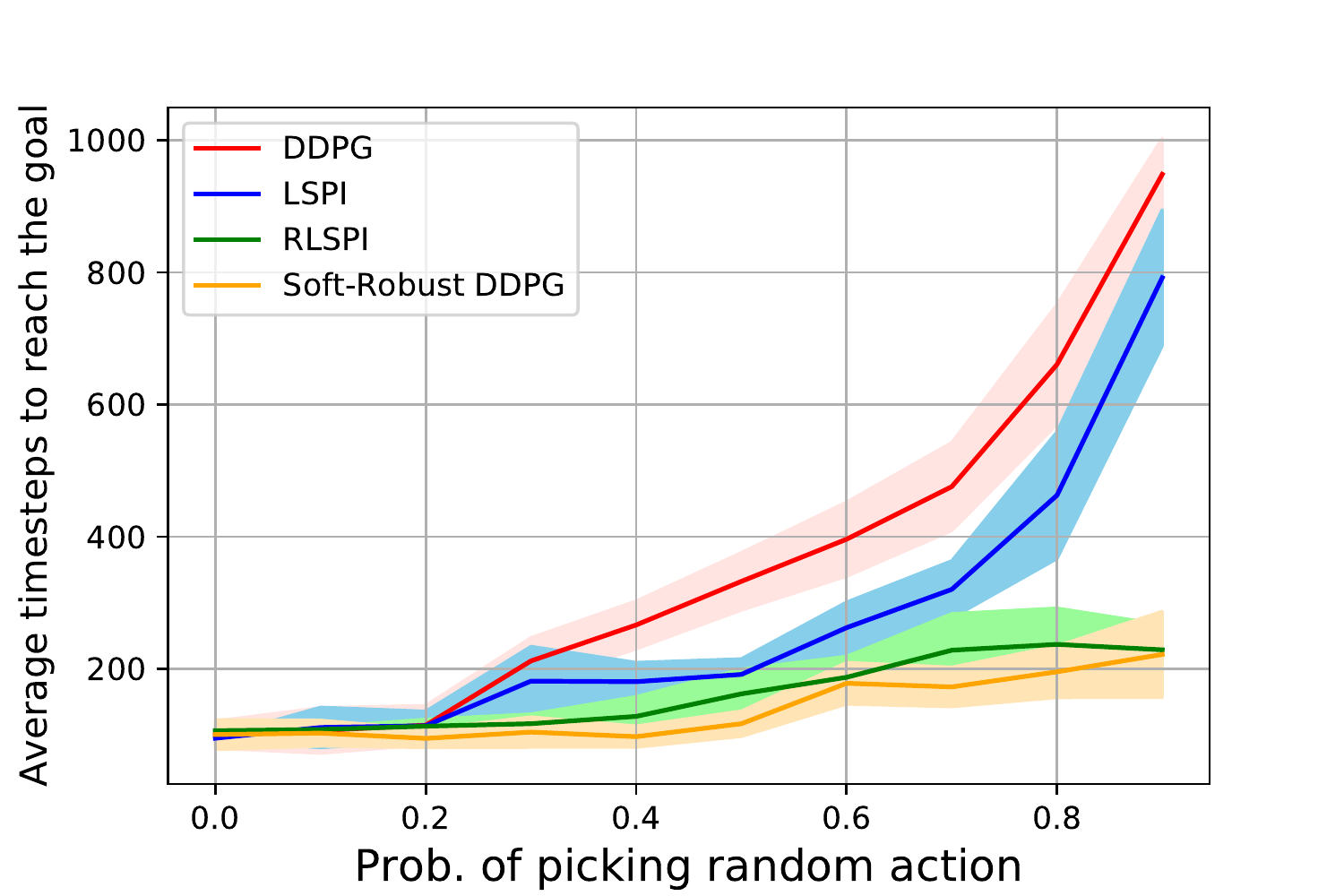}
		\captionof{figure}{MoutainCar}
		\label{fig:mt_failureprob}
	\end{minipage}
	\begin{minipage}{.32\textwidth}
		\centering
		\includegraphics[width=\linewidth]{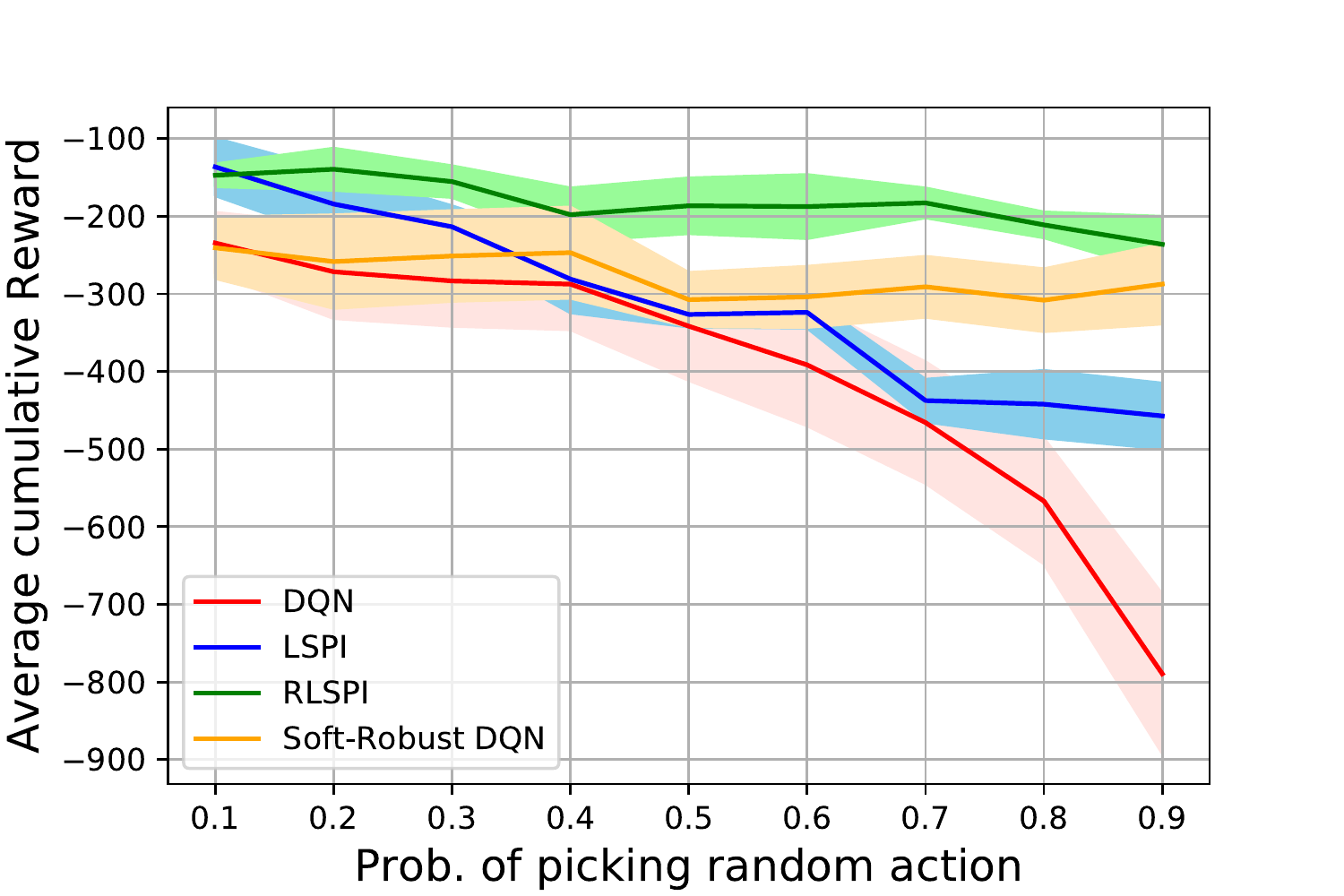}
		\captionof{figure}{Acrobot}
		\label{fig:acro_failureprob}
	\end{minipage}
	\centering
	\begin{minipage}{.32\textwidth}
		\centering
		\includegraphics[width=\linewidth]{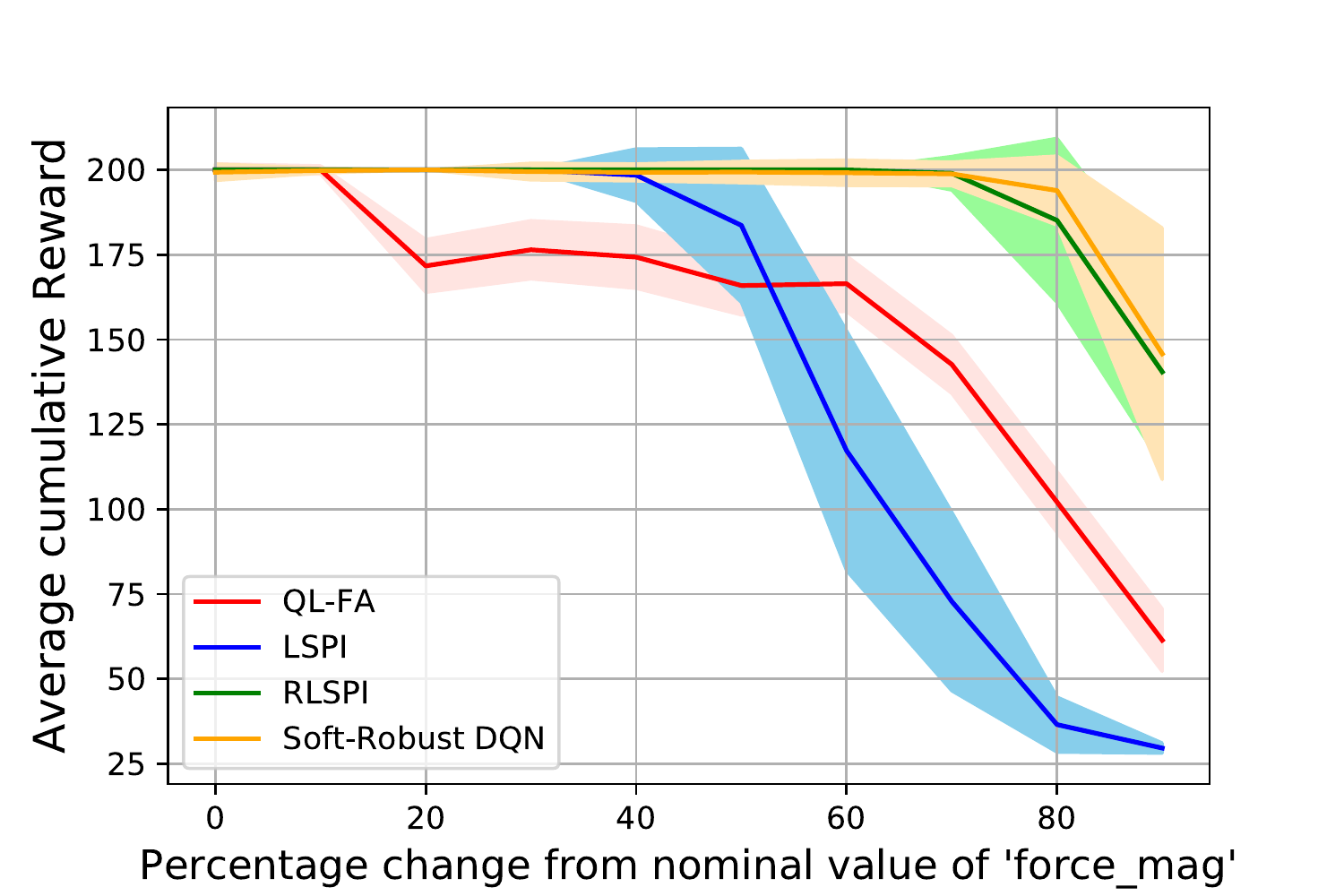}
		\captionof{figure}{CartPole}
		\label{fig:cart_force-in}
	\end{minipage}
	\begin{minipage}{.32\textwidth}
		\centering
		\includegraphics[width=\linewidth]{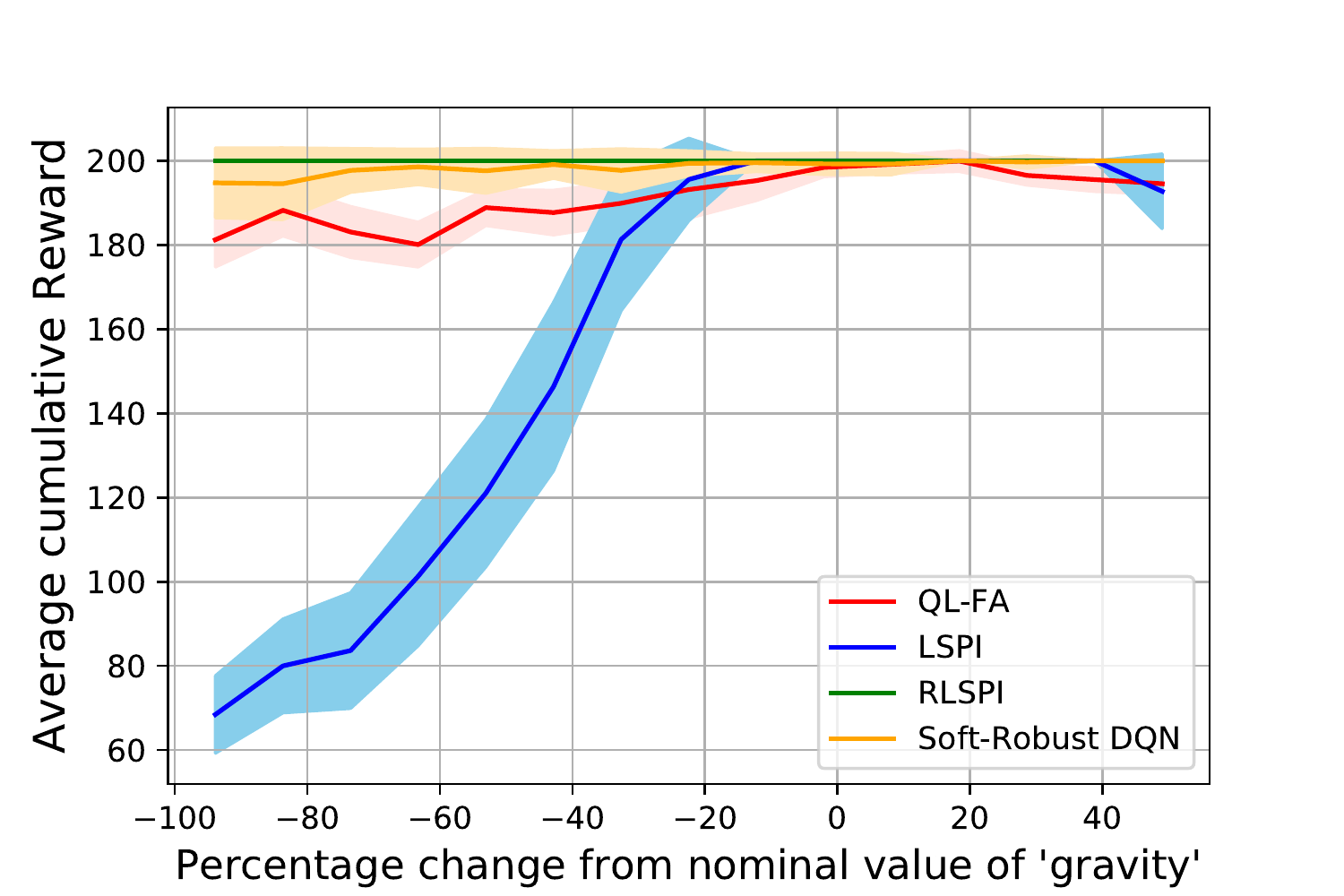}
		\captionof{figure}{CartPole}
		\label{fig:cart_gravity-in}
	\end{minipage}
	\begin{minipage}{.32\textwidth}
		\centering
		\includegraphics[width=\linewidth]{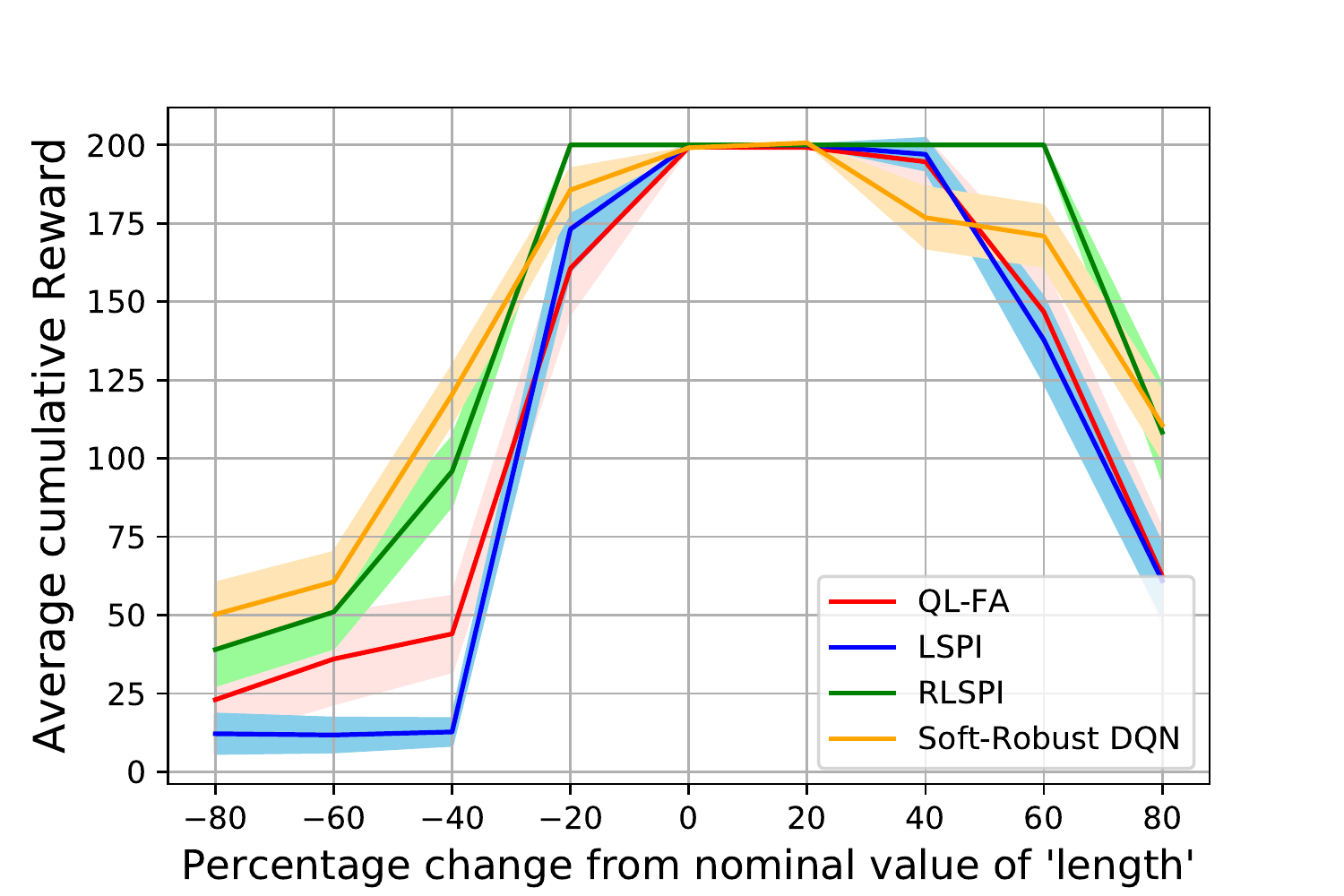}
		\captionof{figure}{CartPole}
		\label{fig:cart_pole-in}
	\end{minipage}
\end{figure*}

We implemented our RLSPI algorithm using the MushroomRL library \citep{deramo2020mushroomrl}, and evaluated its performance against  Q-learning algorithm for an environment with discrete action space, deep deterministic policy gradient (DDPG) \citep{lillicrapHPHETS15} algorithm for continuous action space, and LSPI algorithm  \citep{lagoudakis2003least}.  For comparing with the performance of our RLSPI algorithm against another robust RL algorithm, we implemented the  soft-robust algorithms proposed in \citep{derman2018soft} which use deep neural networks for function approximation.

We chose a spherical uncertainty set with a radius $r$. For such a set $\widehat{\mathcal{U}}$, a closed form solution of $\sigma_{\widehat{\mathcal{U}}}(\Phi w)$ can be computed for faster simulation.  We note that in all the figures shown below, the quantity in the vertical axis is averaged over $100$ runs, with the thick line showing the averaged value and the band around shows the $\pm 0.5$ standard deviation. These figures act as the performance criteria for comparing results. We provide more details and additional experiment results in Section \ref{sec:experiments} of supplementary. 

We used the CartPole, MountainCar, and Acrobot environments from  OpenAI Gym \citep{brockman2016openai}. We trained LSPI algorithm and our RLSPI algorithm on these environments with nominal parameters (default parameters in OpenAI Gym \citep{brockman2016openai}). We also trained Q-learning with linear function approximation and soft-robust deep Q-network(DQN) \citep{derman2018soft} algorithms on CartPole environment, DQN and soft-robust DQN \citep{derman2018soft} algorithms on Acrobot environment, and DDPG and soft-robust DDPG \citep{derman2018soft} algorithms on MountainCar environment. Then, to evaluate  the robustness of the polices obtained, we changed the parameters of these environments and tested the performance of the learned polices on the perturbed environment. 

In Figures \ref{fig:cart_failprob}-\ref{fig:acro_failureprob}, we show the robustness against  action perturbations. In real-world settings, due to model mismatch or noise in the environments, the resulting action  can be different from the intended action. We model this by picking a random action with some probability at each time step. Figure \ref{fig:cart_failprob} shows the  change in  the average episodic reward against the probability of picking a random action for the CartPole environment.  Figure \ref{fig:mt_failureprob} shows the average number of time steps to reach the goal in the MountainCar environment. Figure \ref{fig:acro_failureprob} shows the average episodic reward in the Acrobot environment. In all three cases,  RLSPI algorithm shows robust performance against the perturbations.

Figures \ref{fig:cart_force-in}-\ref{fig:cart_pole-in} shows the test performance on CartPole, by changing the parameters  \emph{force$\_$mag} (external force disturbance), \emph{gravity}, \emph{length} (length of pole on the cart). The nominal values of these parameters  are $10, 9.8$, and  $0.5$ respectively. RLSPI again exhibits robust performance.  

The performance of our RLSPI algorithm is consistently superior to that of the non-robust algorithms. Moreover, the performance of RLSPI algorithm is  comparable  with that of the soft-robust algorithms \citep{derman2018soft}, even though the latter uses deep neural networks for function approximation while our algorithm uses only linear function approximation architecture. We also would like to emphasize that our work gives provable guarantees for the policy learned by  the algorithm whereas \citep{derman2018soft} does not provide any such guarantees.

\section{Conclusion and Future Work}
\label{sec:conclusion}
We have presented an online model-free reinforcement learning algorithm to learn control policies that are robust to the parameter uncertainties of the model, for system with large state spaces. While there have been  interesting empirical works on robust deep RL using neural network, they only provide  convergence guarantees to a local optimum. Different from such empirical works, we proposed a learning based robust policy iteration algorithm called RLSPI algorithm with explicit theoretical guarantees on the performance of the learned policy. To the best of our knowledge, this is the first work that presents  model-free reinforcement learning algorithm  with function approximation for learning the optimal robust policy.  We also empirically evaluated the performance of our RLSPI algorithm on standard benchmark RL problems.

In  future, we plan to extend our theoretical results to nonlinear function approximation architectures. We also plan to characterize the  sample complexity of robust reinforcement learning algorithms. Extending offline RL approaches to robust setting is another research area that we plan to pursue.



\bibliography{RSLPI-References} 
\bibliographystyle{icml2021}
\newpage \onecolumn

\section*{Appendix}

\appendix

\section{Proofs of the Results in Section \ref{sec:RTB-Ber}}
\subsection{Proof of Proposition \ref{prop:informal-cont-fp-1}}
We first restate Proposition \ref{prop:informal-cont-fp-1} formally and then give the proof. 
\begin{proposition}
\label{prop:cont-fp-1}
(i) For any $V_{1}, V_{2} \in \mathbb{R}^{|\mathcal{S}|}$ and $\lambda \in [0, 1),$ $\|{T}^{(\lambda)}_\pi V_{1} -  {T}^{(\lambda)}_\pi V_{2} \|_{\infty} \leq \frac{\alpha (1 - \lambda)}{(1 - \alpha \lambda)} \|V_{1} -  V_{2} \|_{\infty}$.  So, ${T}^{(\lambda)}_\pi$ is a contraction in   sup norm for any $\alpha \in (0, 1), \lambda \in [0, 1)$.  \\ 
(ii)  The robust value function $V_{\pi}$ is the unique fixed point of ${T}^{(\lambda)}_\pi$, i.e., ${T}^{(\lambda)}_\pi  V_\pi = V_\pi,$ for all $ \alpha \in(0, 1)$ and  $\lambda \in [0, 1).$
\end{proposition}
\begin{proof}
From \eqref{eq:TLambda-Ber-1} we have \begin{align*}
\| {T}^{(\lambda)}_\pi (V_1 ) - {T}^{(\lambda)}_\pi (V_2 ) \|_{\infty}&= \| (1-\lambda) \sum_{m=0}^\infty \lambda^{m} (T^{m+1}_{\pi}(V_1)-T^{m+1}_{\pi}(V_2)) \|_\infty\\ &\leq (1-\lambda) \sum_{m=0}^\infty \lambda^{m} \| T^{m+1}_{\pi}(V_1)-T^{m+1}_{\pi}(V_2) \|_\infty
\\ &\stackrel{(a)}{\leq} (1-\lambda) \sum_{m=0}^\infty \lambda^{m} \alpha^{m+1} \| V_1-V_2 \|_\infty = \frac{\alpha (1 - \lambda)}{(1 - \alpha \lambda)} \|V_{1} -  V_{2} \|_{\infty}
\end{align*} where $(a)$ follows since $T_{\pi}$ is  a contraction operator with contraction modulus $\alpha$. This proves $(i)$. Since $V_\pi$ is the unique fixed point of $T_\pi$, $(ii)$ directly follows from $(i)$ and the Banach Fixed Point Theorem \citep[Theorem 6.2.3]{Puterman05}.
\end{proof}

\subsection{Proof of Proposition \ref{prop:informal-cont-fp-2}}
We first restate Proposition \ref{prop:informal-cont-fp-2} formally and then give the proof. 
\begin{proposition}
\label{prop:cont-fp-2}
Under Assumption \ref{as:offpolicy}, for any $V_{1}, V_{2} \in \mathbb{R}^{|\mathcal{S}|}$ and $\lambda \in [0, 1),$ \\ $\|\Pi {T}^{(\lambda)}_\pi V_{1} -  \Pi {T}^{(\lambda)}_\pi V_{2} \|_{d} \leq \frac{\beta (1 - \lambda)}{(1 - \beta \lambda)} \|V_{1} -  V_{2} \|_{d}$.  So, $\Pi {T}^{(\lambda)}_\pi$ is a contraction mapping in $d$-weighted Euclidean norm for any $\beta \in (0, 1), \lambda \in [0, 1)$. 
\end{proposition}
\begin{proof}
From \eqref{eq:TLambda-Ber-1} we have \begin{align*}
\| {T}^{(\lambda)}_\pi (V_1 ) - {T}^{(\lambda)}_\pi (V_2 ) \|_{d}&= \| (1-\lambda) \sum_{m=0}^\infty \lambda^{m} ( T^{m+1}_{\pi}(V_1)- T^{m+1}_{\pi}(V_2)) \|_d\\ &\leq (1-\lambda) \sum_{m=0}^\infty \lambda^{m} \| T^{m+1}_{\pi}(V_1)- T^{m+1}_{\pi}(V_2) \|_d
\\ &\stackrel{(a)}{\leq} (1-\lambda) \sum_{m=0}^\infty \lambda^{m} \beta^{m+1} \| V_1-V_2 \|_d = \frac{\beta (1 - \lambda)}{(1 - \beta \lambda)} \|V_{1} -  V_{2} \|_{d}
\end{align*} where $(a)$ follows since $\Pi T_{\pi}$ is  a contraction in the $d$-weighted Euclidean norm with contraction modulus $\beta$  \citep[Corollary 4]{tamar2014scaling}. From \citep{tsitsiklis1997analysis}, $\Pi$ is a nonexpansive mapping in the $d$-weighted Euclidean norm. So, $\Pi {T}^{(\lambda)}_\pi$ has the stated proptery. 
\end{proof}

\subsection{Derivation of  \eqref{eq:rpvi-11}}
 Given $w_{k}$,  $w_{k+1}$ which satisfies the equation $\Phi w_{k+1} = \Pi T^{(\lambda)}_{\pi} \Phi w_{k}$ can be written as the solution of the minimization problem $w_{k+1} = \arg \min_{w} \| \Phi w  -  T^{(\lambda)}_{\pi} \Phi w_{k}  \|_d^2$. 
Taking gradient w.r.t. $w$ and equating to zero, we get $\Phi^{\top} D  (\Phi w  -  T^{(\lambda)}_\pi \Phi w_{k})  =0$, which implies $ w_{k+1} = (\Phi^{\top} D \Phi)^{-1} \Phi^{\top} D   T^{(\lambda)}_\pi \Phi w_{k}$. This can  be written as 
\begin{align*}
w_{k+1} &= (\Phi^{\top} D \Phi)^{-1} \Phi^{\top} D   T^{(\lambda)}_{\pi} \Phi w_{k} = w_{k} + (\Phi^{\top} D \Phi)^{-1} (\Phi^{\top} D   T^{(\lambda)}_{\pi} \Phi w_{k}  - \Phi^{\top} D \Phi w_{k}) \nonumber \\
&= w_{k} + (\Phi^{\top} D \Phi)^{-1} \Phi^{\top} D (T^{(\lambda)}_{\pi} \Phi w_{k}  - \Phi w_{k}).
\end{align*}
\subsection{Derivation of  \eqref{eq:TLambda-Ber-expanded}}
We have
\begin{align}
T V & = r_{\pi} + \alpha P^o_\pi V + \alpha \sigma_{\U_\pi} (V) \label{eq:base_T} \\
T^{2} V &= r_{\pi} + \alpha P^o_\pi (TV) + \alpha \sigma_{\U_\pi} (T V) \nonumber \\
&= (1 + \alpha P) r_{\pi} + (\alpha P^o_\pi)^{2} V   + \alpha (\alpha P^o_\pi)  \sigma_{\U_\pi} (V) + \alpha \sigma_{\U_\pi} (T V). \nonumber
\end{align} Suppose we have an induction hypothesis for $T^m V$, i.e.,
\begin{align*}
T^{m}V &= \sum^{m-1}_{k=0} (\alpha P^o_\pi)^{k} r_{\pi} + (\alpha P^o_\pi)^{m} V + \alpha \sum^{m-1}_{k=0} (\alpha P^o_\pi)^{k} \sigma_{\U_\pi} (T^{(m-1-k)} V).
\end{align*} Now, to verify an induction step, observe that,
\begin{align*}
 &T^{m+1}V= \sum^{m-1}_{k=0} (\alpha P^o_\pi)^{k} r_{\pi} + (\alpha P^o_\pi)^{m} TV + \alpha \sum^{m-1}_{k=0} (\alpha P^o_\pi)^{k} \sigma_{\U_\pi} (T^{(m-1-k)} TV) \\
&\stackrel{(a)}{=} \sum^{m-1}_{k=0} (\alpha P^o_\pi)^{k} r_{\pi} + (\alpha P^o_\pi)^{m} (r + \alpha P^o_\pi V + \alpha \sigma_{\U_\pi} (V)) + \alpha \sum^{m-1}_{k=0} (\alpha P^o_\pi)^{k} \sigma_{\U_\pi} (T^{(m-1-k)} TV) \\
&= \sum^{m}_{k=0} (\alpha P^o_\pi)^{k} r_{\pi} + (\alpha P^o_\pi)^{m+1} V + \alpha \sum^{m}_{k=0} (\alpha P^o_\pi)^{k} \sigma_{\U_\pi} (T^{(m-k)} V)
\end{align*} where $(a)$ follows from \eqref{eq:base_T}. Now, \eqref{eq:TLambda-Ber-expanded} directly follows from \eqref{eq:TLambda-Ber-1}.


\section{Proofs of the Results in Section \ref{sec:rlspe-1}}
\subsection{Proof of Proposition \ref{prop:fixedpoint-1} }
\begin{proof} With $\widehat{\U}_\pi = {\U}_\pi$,  we have, 
\begin{align*}
\widetilde{T}^{(\lambda)}_\pi (V_{\pi} )& =  (1-\lambda) \sum_{m=0}^\infty \lambda^{m} \left[\sum_{t=0}^{m} (\alpha P^{o}_\pi)^t {r}_\pi+ \alpha \sum_{t=0}^{m} (\alpha P^{o}_\pi)^t  ~ {\sigma}_{{\U_\pi}} (V_{\pi}) + (\alpha P^{o}_\pi)^{m+1} V_{\pi} \right] \\
&\stackrel{(a)}{=}  (1-\lambda) \sum_{m=0}^\infty \lambda^{m} \left[\sum_{t=0}^{m} (\alpha P^{o}_\pi)^t  ({r}_\pi + \alpha {\sigma}_{{\U_\pi}} (V_{\pi}) ) + (\alpha P^{o}_\pi)^{m+1} V_{\pi} \right] \\
&=(1-\lambda) \sum_{m=0}^\infty \lambda^{m} \left[\sum_{t=0}^{m} (\alpha P^{o}_\pi)^t   (I -\alpha P^{o}_\pi ) V_{\pi} + (\alpha P^{o}_\pi)^{m+1} V_{\pi} \right] \\
&\stackrel{(b)}{=}   (1 - \lambda)  \sum_{m=0}^\infty \lambda^{m} V_{\pi}  = V_{\pi}.
\end{align*}
Here, to get (a), we wrote $r_{\pi} + \alpha \sigma_{\mathcal{U}_{\pi}} (V_{\pi})= (I -\alpha P^{o}_\pi ) V_{\pi}$  since $T_{\pi} V_{\pi} = V_{\pi}$.  $(b)$ is from the telescopic sum of the previous equation. 
\end{proof}




\subsection{Proof of Theorem \ref{prop:approx-td-op-1} }

We have the following lemma which is similar to Lemma 4.2 in \citep{roy2017reinforcement}. 
\begin{lemma}
\label{lem:sigma_d_u_uhat}
For any vector $ V\in \R^{|\mathcal{S}|}$ and for all $(s, a) \in \mathcal{S} \times \mathcal{A},$  \[| \sigma_{\widehat{\U}_{s,a}} (V) - \sigma_{{\U}_{s,a}} (V)  | \leq \rho \|V\|_d.\] 
\end{lemma}
\begin{proof}
First note that, for any $x, y \in \R^{|\mathcal{S}|}$ we have \begin{equation}
x^{\top} y ~~\leq~~ ({x^\top D y})/ {d_{\text{min}}} ~~\stackrel{(a)}{\leq}~~ ({\|x\|_d \|y\|_d})/{d_{\text{min}}}, \label{eq:cs_d}
\end{equation}
where $(a)$ follows from Cauchy-Schwarz inequality with respect to $\|\cdot\|_d$ norm.

Consider any $p\in \widehat{\U}_{s,a}$ and $q \in {\U}_{s,a} \setminus \widehat{\U}_{s,a}$. For any $V\in \R^{|\mathcal{S}|}$, we have \begin{align*}
&p^\top V = q^\top V + (p-q)^\top V \geq  \sigma_{{\U}_{s,a}} (V) + (p-q)^\top V \geq \sigma_{{\U}_{s,a}} (V)+ \min_{x\in\widehat{\U}_{s,a}}  \min_{y\in {\U}_{s,a} \setminus \widehat{\U}_{s,a}} (x-y)^\top V\\
&= \sigma_{{\U}_{s,a}} (V)- \max_{x\in\widehat{\U}_{s,a}} \max_{y\in {\U}_{s,a} \setminus \widehat{\U}_{s,a}}  (y-x)^\top V \stackrel{(b)}{\geq}  \sigma_{{\U}_{s,a}} (V)- \max_{x\in\widehat{\U}_{s,a}} \max_{y\in {\U}_{s,a} \setminus \widehat{\U}_{s,a}}  ({\|y-x\|_d \|V\|_d})/{d_{\text{min}}} \\
&\geq \sigma_{{\U}_{s,a}} (V) - \rho_{s,a} \|V\|_d,
\end{align*} 
where $(b)$ follows from \eqref{eq:cs_d}. By taking infimum on both sides with respect to $p\in\widehat{\U}_{s,a}$,  we get, 
\[\sigma_{\widehat{\U}_{s,a}} (V) \geq \sigma_{{\U}_{s,a}} (V) - \rho_{s,a} \|V\|_d.\] We can also get $\sigma_{{\U}_{s,a}} (V) \geq \sigma_{\widehat{\U}_{s,a}} (V)  - \rho_{s,a} \|V\|_d$ by similar arguments. Combining, we get,  \[| \sigma_{\widehat{\U}_{s,a}} (V) - \sigma_{{\U}_{s,a}} (V)  | \leq \rho_{s,a} \|V\|_d.\] Since $\rho = \max_{s, a} \rho_{s,a},$ we get the desired result. 
\end{proof}

We will use the following result which follows directly from   \citep[Lemma 1]{tsitsiklis1997analysis}. 
\begin{lemma}
\label{lem:po}
Under Assumption \ref{as:offpolicy}, for any $V$, we have $\|P^o_{\pi_{e}} V \|_d \leq \| V \|_d.$
\end{lemma}

\begin{remark}\label{remark_op}
	The inequality in the Assumption \ref{as:offpolicy} can be written as, $\alpha P^{o}_{s,\pi(s)}(s') + \alpha U_{s,\pi(s)}(s') \leq \beta P^{o}_{s,\pi_{e}(s)}(s')$. From this, we can conclude that $\alpha U_{s,\pi(s)}(s') \leq \beta P^{o}_{s,\pi_{e}(s)}(s')$.
\end{remark}

Next, we show the following. 
\begin{lemma}
\label{lemma_sigma_d_}
For any  $ V_{1}, V_{2} \in \R^{|\mathcal{S}|}$,  
\begin{align*}
  \|{\sigma}_{\widehat{\U}_\pi} (V_{1}) - {\sigma}_{\widehat{\U}_\pi} (V_{2}) \|_{d} \leq \left(\frac{\beta}{\alpha}+\rho\right) \|V_{1} - V_{2}\|_{d}.
\end{align*}
\end{lemma}
\begin{proof}
	For any $s\in\Ss$ we have  
	\begin{align}
	& \sigma_{{\widehat{\U}_{s,\pi(s)}}} (V_2) - \sigma_{{\widehat{\U}_{s,\pi(s)}}} (V_1) = \inf_{q \in {\widehat{\U}_{s,\pi(s)}}} q^\top V_2 - \inf_{\tilde{q} \in {\widehat{\U}_{s,\pi(s)}}} \tilde{q}^\top V_1 = \inf_{q \in {\widehat{\U}_{s,\pi(s)}} } \sup_{\tilde{q} \in {\widehat{\U}_{s,\pi(s)}}} q^\top V_2 - \tilde{q}^\top V_1 \nonumber  \\
	&\geq \inf_{q \in {\widehat{\U}_{s,\pi(s)}} } q^\top (V_2 - V_1) = \sigma_{\widehat{\U}_{s,\pi(s)}} (V_2 - V_1) \stackrel{(a)}{\geq} \sigma_{{\U}_{s,\pi(s)}} (V_2 - V_1) - \rho \| V_1 - V_2 \|_d,
	\label{eq:lemma_sigma_d-1}
	\end{align}
	where $(a)$ follows from Lemma \ref{lem:sigma_d_u_uhat}. By definition, for any arbitrary $\epsilon > 0$, there exists a $U_{s, \pi(s)}  \in \mathcal{U}_{s, \pi(s)} $ such that
	\begin{align}
	\label{eq:lemma_sigma_d-2}
	U^{\top}_{s, \pi(s)} (V_{2} - V_{1}) - \epsilon \leq \sigma_{{\mathcal{U}_{s,\pi(s)}}} (V_2 - V_1).
	\end{align}
	Using \eqref{eq:lemma_sigma_d-2} and \eqref{eq:lemma_sigma_d-1},
	\begin{align}
	\alpha (\sigma_{{\widehat{\U}_{s,\pi(s)}}} (V_1) - \sigma_{{\widehat{\U}_{s,\pi(s)}}} (V_2) ) &\leq \alpha U^{\top}_{s, \pi(s)}  (V_{1} - V_{2}) + \alpha \epsilon  + \rho \alpha \| V_1 - V_2 \|_d \nonumber \\
	& \leq \alpha U^{\top}_{s, \pi(s)}  |(V_{2} - V_{1})| + \alpha \epsilon  + \rho \alpha \| V_1 - V_2 \|_d \nonumber \\
	&\stackrel{(b)}{\leq}  \beta  (P^{o}_{s,  \pi_{e}(s)})^{\top} |(V_{1} - V_{2})|  + \alpha \epsilon + \rho \alpha \| V_1 - V_2 \|_d \label{eq:sigma_u_intermediate}
	\end{align}
	where $(b)$ follows from Remark \ref{remark_op}. Since $\epsilon$ is arbitrary, we get, 
	\[\alpha (\sigma_{{\widehat{\U}_{s,\pi(s)}}} (V_1) - \sigma_{{\widehat{\U}_{s,\pi(s)}}} (V_2) ) \leq  \beta  (P^{o}_{s,  \pi_{e}(s)})^{\top} |(V_{1} - V_{2})| + \rho \alpha \| V_1 - V_2 \|_d .\] By exchanging the roles of $V_1$ and $V_2$, we get $\alpha (\sigma_{{\widehat{\U}_{s,\pi(s)}}} (V_2) - \sigma_{{\widehat{\U}_{s,\pi(s)}}} (V_1) ) \leq  \beta  (P^{o}_{s,  \pi_{e}(s)})^{\top} |(V_{1} - V_{2})| + \rho \alpha \| V_1 - V_2 \|_d $. Combining these and writing compactly in vector form, we get 
	\[\alpha |\sigma_{{\widehat{\U}_{\pi}}} (V_1) - \sigma_{{\widehat{\U}_{\pi}}} (V_2) | \leq  \beta  P^{o}_{\pi_{e}}|V_{1} - V_{2}| + \rho \alpha \| V_1 - V_2 \|_d \textbf{1},\]
	where $\textbf{1} = (1, 1, \ldots, 1)^{\top}$, an $|\Ss|$-dimensional unit vector. Since $\alpha |\sigma_{{\widehat{\U}_{\pi}}} (V_1) - \sigma_{{\widehat{\U}_{\pi}}} (V_2) | \geq 0$, by the property of the norm, we get \begin{align*}
	\alpha \|\sigma_{{\widehat{\U}_{\pi}}} (V_1) - \sigma_{{\widehat{\U}_{\pi}}} (V_2) \|_d &\leq \| \beta  P^{o}_{\pi_{e}} |V_{1} - V_{2}| + \rho \alpha \| V_1 - V_2 \|_d \textbf{1} \|_d \\&\stackrel{(c)}{\leq} \beta  \|P^{o}_{\pi_{e}}|V_{1} - V_{2}| \|_d + \rho \alpha \| V_1 - V_2 \|_d \| \textbf{1} \|_d \\ &\stackrel{(d)}{\leq}  \beta  \|V_{1} - V_{2} \|_d + \rho \alpha \| V_1 - V_2 \|_d
	\end{align*} 
	where $(c)$ follows from triangle inequality and $(d)$  from Lemma \ref{lem:po} and from the fact that and $\| \textbf{1} \|_d=1$. 
	 Dividing by $\alpha$ and rearranging, we get the desired result.
\end{proof}

\begin{proof}[\textbf{Proof of Theorem \ref{prop:approx-td-op-1}}]
	We first observe
	\begin{align}
	\|\alpha P^{o}_{\pi}  |V|  \|_{d} &\stackrel{(a)}{\leq} \|\beta P^{o}_{\pi_{e}} |V| \|_{d}  \stackrel{(b)}{\leq} \beta \|V\|_{d}, \label{eq:repeat_}
	\end{align} where $(a)$ follows from the Assumption \ref{as:offpolicy} and $(b)$ from the Lemma \ref{lem:po}.
	
	Notice that for any finite $t$ we have,  
	\begin{align}
	\| (\alpha P^{o}_{\pi})^{t}  |V|  \|_{d}  = \| \alpha P^{o}_{\pi}  (\alpha P^{o}_{\pi})^{t-1}  |V|  \|_{d}  &\stackrel{(c)}{\leq} \beta \| (\alpha P^{o}_{\pi})^{t-1}  |V|  \|_{d} 
	\end{align} 
	where $(c)$ follows from \eqref{eq:repeat_}.
	Using this repeatedly, we get,
	\begin{align}
	\label{eq:repeat_final_}
	\| (\alpha P^{o}_{\pi})^{t}  |V|  \|_{d} \leq \beta^{t} \|V\|_{d}. 
	\end{align}
	Now,
	\begin{align}
	&\|\widetilde{T}^{(\lambda)}_\pi (V_{1}) - \widetilde{T}^{(\lambda)}_\pi  (V_{2}) \|_{d} \nonumber \\
	&= \|(1-\lambda) \sum_{m=0}^\infty \lambda^{m} [\alpha \sum_{t=0}^{m} (\alpha P^{o}_{\pi})^t  ~ ({\sigma}_{\widehat{\U_\pi}} (V_{1}) - {\sigma}_{\widehat{\U_\pi}} (V_{2}) )+ (\alpha P^{o}_{\pi})^{m+1} (V_{1} - V_{2})] \|_{d} \nonumber \\
	&\leq (1-\lambda) \sum_{m=0}^\infty \lambda^{m} [\alpha \sum_{t=0}^{m}   ~ \|(\alpha P^{o}_{\pi})^{t}|{\sigma}_{\widehat{\U_\pi}} (V_{1}) - {\sigma}_{\widehat{\U_\pi}} (V_{2})| \|_{d}+ \|( \alpha P^{o}_{\pi})^{m+1} |V_{1} - V_{2}| \|_{d} ] \nonumber \\
	&\stackrel{(d)}{\leq}  (1-\lambda) \sum_{m=0}^\infty \lambda^{m} [\alpha \sum_{t=0}^{m} (\beta)^t  ~ \|{\sigma}_{\widehat{\U_\pi}} (V_{1}) - {\sigma}_{\widehat{\U_\pi}} (V_{2}) \|_{d}+ \beta^{m+1} \|V_{1} - V_{2} \|_{d} ] \nonumber \\ 
	&\stackrel{(e)}{\leq}   (1-\lambda) \sum_{m=0}^\infty \lambda^{m} [(\beta+\rho\alpha) \sum_{t=0}^{m} (\beta)^{t}  ~ \|(V_{1} - V_{2}) \|_{d}+ \beta^{m+1} \|(V_{1} - V_{2}) \|_{d}] \label{eq:for_l_g_pi} \\ 
	&= \left[\frac{(\beta+\rho\alpha)}{(1 - \beta  )}\left(1-\frac{\beta(1-\lambda) }{(1 - \beta  \lambda)}\right) +\frac{\beta(1-\lambda) }{(1 - \beta  \lambda)}\right] \|(V_{1} - V_{2})\|_{d} \nonumber \\
	&= \frac{\beta(2-\lambda) + \rho\alpha  }{(1 - \beta  \lambda)}\|(V_{1} - V_{2})\|_{d}, \label{eq:contract_coeff_calc}
	\end{align} where $(d)$ follows from \eqref{eq:repeat_final_} and $(e)$ follows from Lemma \ref{lemma_sigma_d_}. 
	
	From \citep{BerBook12}, $\Pi$ is a non-expansive operator in $\|\cdot\|_{d}$. Thus,
	\begin{align}
	    \|\Pi \widetilde{T}^{(\lambda)}_\pi V_{1} -  \Pi \widetilde{T}^{(\lambda)}_\pi V_{2} \|_{d} \leq  \|\widetilde{T}^{(\lambda)}_\pi (V_{1}) - \widetilde{T}^{(\lambda)}_\pi  (V_{2}) \|_{d} \leq c(\alpha, \beta, \rho, \lambda) ~ \|V_{1} - V_{2}\|_{d}.
	\end{align}
	where $c(\alpha, \beta, \rho, \lambda) = {(\beta(2-\lambda) + \rho\alpha ) }/{(1 - \beta  \lambda)}$. 	This concludes the proof of getting \eqref{eq:Ttilde-contraction-1}.
	
	
	For proving \eqref{eq:Ttilde-contraction-3},  first denote the operator $\widetilde{T}^{(\lambda)}_\pi $ as $\bar{T}^{(\lambda)}_\pi$ when  $\widehat{\U}_\pi = {\U}_\pi$. Now observe that 
	\begin{align}
	\| \bar{T}^{(\lambda)}_\pi (V ) - \widetilde{T}^{(\lambda)}_\pi (V ) \|_d &= \| (1-\lambda) \sum_{m=0}^\infty \lambda^{m} [ \alpha \sum_{t=0}^{m} (\alpha P^{o}_\pi)^t  ~ ({\sigma}_{{\U_\pi}} (V)-{\sigma}_{\widehat{\U}_\pi} (V)) ] \|_d \nonumber \\
	&\leq  (1-\lambda) \sum_{m=0}^\infty \lambda^{m} [ \alpha \sum_{t=0}^{m} \| (\alpha P^{o}_\pi)^t  ~ |({\sigma}_{{\U_\pi}} (V)-{\sigma}_{\widehat{\U}_\pi} (V))| \|_d ]  \nonumber \\
	&\stackrel{(f)}{\leq}  (1-\lambda) \sum_{m=0}^\infty \lambda^{m} [ \alpha \sum_{t=0}^{m} \beta^t \|{\sigma}_{{\U_\pi}} (V)-{\sigma}_{\widehat{\U}_\pi} (V) \|_d ]  \nonumber \\
	&\stackrel{(g)}{\leq} (1-\lambda) \sum_{m=0}^\infty \lambda^{m}  [\frac{(1 - \beta^{m+1})}{1-\beta}  ~ \alpha \rho \| V\|_d  ]  =  \frac{\alpha \rho \| V\|_d }{1-\beta\lambda} \label{eq:pi_1}
	\end{align} 
	where $(f)$ follows \eqref{eq:repeat_final_} and $(g)$ from Lemma \ref{lem:sigma_d_u_uhat}.

	Now, 
	\begin{align*}
	\| V_\pi - \Phi w_\pi \|_d &\leq \| V_\pi -  \Pi V_\pi  \|_d + \| \Pi V_\pi - \Phi w_\pi \|_d \\
	&\stackrel{(h)}{=} \| V_\pi -  \Pi V_\pi  \|_d + \| \Pi \bar{T}^{(\lambda)}_\pi V_\pi -  \Pi \widetilde{T}^{(\lambda)}_\pi \Phi w_\pi \|_d \\
	&\stackrel{(i)}{\leq} \| V_\pi -  \Pi V_\pi  \|_d +  \| \Pi \bar{T}^{(\lambda)}_\pi V_\pi -  \Pi \widetilde{T}^{(\lambda)}_\pi V_\pi  \|_d + \|\Pi \widetilde{T}^{(\lambda)}_\pi V_\pi -  \Pi \widetilde{T}^{(\lambda)}_\pi \Phi w_\pi \|_d \\
	&\stackrel{(j)}{\leq} \| V_\pi -  \Pi V_\pi  \|_d  + \frac{\beta \rho \| V_{\pi}\|_d}{1-\beta\lambda} + c(\alpha, \beta, \rho, \lambda) \| V_\pi - \Phi w_\pi \|_d 
	\end{align*}
	We get $(h)$ because $\bar{T}^{(\lambda)}_\pi V_\pi = V_{\pi}$ from Proposition \ref{prop:fixedpoint-1} and  $\Pi \widetilde{T}^{(\lambda)}_\pi \Phi w_\pi = \Phi w_\pi$ by the premise of the proposition, $(i)$ by triangle inequality, $(j)$ from  \eqref{eq:pi_1} and \eqref{eq:Ttilde-contraction-1}. Rearranging, we get,
	\begin{align}
	   \| V_\pi - \Phi w_\pi \|_d \leq \frac{1}{1 - c(\alpha, \beta, \rho, \lambda)} \left(\| V_\pi -  \Pi V_\pi  \|_d +  \frac{\beta \rho \| V_{\pi}\|_d}{1-\beta\lambda} \right). 
	\end{align}
	
	This completes the proof of Theorem \ref{prop:approx-td-op-1}.
\end{proof}

\subsection{Derivation of \eqref{eq:rpvi-approx-matrix-1}}

For any bounded mapping $W$, observe that \[ \sum_{t=0}^\infty(\alpha \lambda P^o_\pi)^t W =  (1-\lambda) \sum_{m=0}^\infty \lambda^m \sum_{t=0}^m (\alpha P^o_\pi)^t W,  \] yielded by exchanging summations. Using this observation with equations \eqref{eq:TLambda-1} and \eqref{eq:rpvi-approx-1} we get \eqref{eq:rpvi-approx-matrix-1}.

\subsection{Proof of Theorem \ref{thm:RLSPE-convergence-1}  }
To prove this theorem, we will use the following result from \citep{nedic2003least}. Let $\|\cdot\|$ denote the standard Euclidean norm. 
\begin{proposition}\citep[Proposition 4.1]{nedic2003least}   
	\label{thm:NB02-prop}
	Consider a sequence $\{x_{t}\}$ generated by the update equation
	\[x_{t+1} = x_{t} + \gamma_{t} (h_{t}(x_{t}) + e_{t}),\] 
	where $h_{t} : \R^{n} \rightarrow \R^{n}$, $\gamma_{t}$ is a positive deterministic stepsize, and $e_{t}$ is a random noise vector.  Let $f : \R^{n} \rightarrow \R$ be a continously differentiable function. Assume the following:\\
	(i) Function $f$ is positive, i.e.,  $f(x) \geq 0$. Also, $f$ has Lipschitz continuous gradient, i.e., there exists some scalar $L>0$ such that
	\begin{align}
	\label{eq:NB-condition0}
	\|\nabla f(x) - \nabla f(y)\| \leq L \|x-y\|, ~~\forall ~x,y \in \mathbb{R}^{n}.
	\end{align}
	(ii) Let $\mathcal{F}_{t} = \{x_{0}, x_{1}, \ldots, x_{t}\}$. There exists positive scalars $c_{1}, c_{2},$ and $c_{3}$ such that
	\begin{align}
	\label{eq:NB-condition1}
	(\nabla f(x_{t}))^{\top} \E[h_{t}(x_{t}) | \mathcal{F}_{t}] &\leq - c_{1} \| \nabla f(x_{t})\|^{2}, ~~\forall~t, \\
	\label{eq:NB-condition2}
	\|\E[e_{t}|\mathcal{F}_{t}] \| &\leq c_{2} \epsilon_{t} (1 + \| \nabla f(x_{t})\| ), ~~\forall~t,\\
	\label{eq:NB-condition3}
	\E[\| h_{t}(x_t) + e_{t}\|^{2} | \mathcal{F}_{t}] &\leq  c_{3}  (1 + \| \nabla f(x_{t})\|^{2} ),~~\forall~t,
	\end{align}
	where $\epsilon_{t}$ is a positive deterministic scalar. 
	
	(ii) The deterministic sequences $\{\gamma_{t}\}$ and $\{\epsilon_{t}\}$ satsify
	\begin{align}
	\label{eq:NB-condition4}
	\sum^{\infty}_{t=0} \gamma_{t} = \infty,~ \sum^{\infty}_{t=0} \gamma^{2}_{t} <  \infty,~ \sum^{\infty}_{t=0} \gamma_{t} \epsilon^{2}_{t} < \infty,~ \lim_{t \rightarrow \infty} \epsilon_{t} = 0.
	\end{align}
	
	Then, with probability 1:\\
	(i) The sequence $\{f(x_{t})\}$ converges.\\
	(ii) The sequence $\{\nabla f(x_{t})\}$ converges to zero.\\
	(iii) Every limit point of $\{x_{t}\}$ is a stationary point of $f$.
\end{proposition}

Using the above result, we now prove Theorem \ref{thm:RLSPE-convergence-1}.
\begin{proof}[\textbf{Proof of Theorem \ref{thm:RLSPE-convergence-1}}]

	We rewrite \eqref{eq:RLSPE-1} as $w_{t+1} = w_{t} + \gamma_{t} (h_{t}(w_{t}) + e_{t}),$ where,
	\begin{align*}
	h_{t}(w_{t}) &= B^{-1} (A w_{t} + b + C(w_{t})), \\
	e_{t} &=  B^{-1}_{t}(A_{t} w_{t} + b_{t} + C_{t}(w_{t})) -B^{-1} (A w_{t} + b + C(w_{t})).
	\end{align*}
	Define
	\[f(w) =  \frac{1}{2} (w - w_{\pi})^{\top} \Phi^{\top} D \Phi (w - w_{\pi}).\] 
	We now verify the conditions given in Proposition \ref{thm:NB02-prop}.
	
	\textit{Verifying \eqref{eq:NB-condition0}:}
	By definition $f(w) \geq 0.$ Also, $\nabla f(w) = \Phi^{\top} D \Phi (w - w_{\pi})$. Hence,
	\begin{align*}
	\| \nabla f(w) - \nabla f(w') \| =  \| \Phi^{\top} D \Phi (w - w') \| \leq \| \Phi^{\top} D \Phi \|_{\text{op}}  \|w - w' \|,
	\end{align*} where $\|\cdot\|_{\text{op}}$ is the operator norm corresponding to the Euclidean space. Set $L= \| \Phi^{\top} D \Phi \|_{\text{op}}$. From $(ii)$ in Assumption \ref{as:offpolicy} and $\Phi$ being full rank, we know that $\Phi^{\top} D \Phi$ is a positive definite matrix and hence $L$ is a positive scalar. This verifies \eqref{eq:NB-condition0}.

	\textit{Verifying \eqref{eq:NB-condition1}:} It is straightforward to verify that $ A w_t + b + C(w_t)  = \Phi^{\top} D (\widetilde{T}^{(\lambda)}_{\pi} \Phi w_{t} - \Phi w_{t})$ by comparing \eqref{eq:rpvi-approx-1}  - \eqref{eq:Cb-1}. Then,
	\begin{align*}
	( \nabla f(w_t))^{\top} \E[h_{t}(w_t)|\F_{t}] &= (w_t - w_{\pi})^{\top}  (Aw_t + b + C(w_t)) \\
	&= (w_t - w_{\pi})^{\top} \Phi^{\top} D (\widetilde{T}^{(\lambda)}_{\pi} \Phi w_{t} - \Phi w_{t}) < 0,
	\end{align*}
	where the last inequality is by using Lemma 9 from  \citep{tsitsiklis1997analysis}  with the fact that, from Theorem \ref{prop:approx-td-op-1}, $\Pi \widetilde{T}^{(\lambda)}_{\pi}$ is a contraction. Since $( \nabla f(w_t))^{\top} \E[h_{t}|\F_{t}]$ is strictly negative, we can find a positive scalar $c_{2}$ that satisfies  \eqref{eq:NB-condition1}. 
	
	\textit{Verifying \eqref{eq:NB-condition2}:} We can write $
	e_{t}  = \Delta^{1}_{t} w_{t} + \Delta^{2}_{t} + \Delta^{3}_{t}(w_{t}),$ where,
	\begin{align*}
	\Delta^{1}_{t} &= B^{-1}_{t} A_{t} - B^{-1}A,~ \Delta^{2}_{t} =  B^{-1}_{t} b_{t} - B^{-1} b,~ \Delta^{3}_{t}(w_{t}) = B^{-1}_{t} C_{t}(w_{t}) - B^{-1} C(w_{t}).
	\end{align*}
	From Proposition 2.1 of \citep{nedic2003least}, we have
	\begin{align}
	\label{eq:verify-c2-s4}
	\|\E[\Delta^{1}_{t}] \| \leq \frac{\bar{c}_{1}}{\sqrt{t+1}}, ~~ \|\E[\Delta^{2}_{t}] \|  \leq \frac{\bar{c}_{2}}{\sqrt{t+1}} ~~\forall~t.
	\end{align}
	So, we will now bound $\E[\Delta^{3}_{t} (w_{t}) |  \F_{t}]$. 
	\begin{align}
	&\|\E[\Delta^{3}_{t} (w_{t})|\F_t] \| = \| \E[B^{-1}_{t} C_{t}(w_{t}) - B^{-1} C(w_{t})|\F_t]\| \nonumber \\
	&=\| \E[B^{-1}_{t} C_{t}(w_{t}) - B^{-1} C_{t}(w_{t}) + B^{-1} C_{t}(w_{t}) - B^{-1} C(w_{t})|\F_t] \| \nonumber \\
	\label{eq:verify-c2-s1}
	&\leq \E[ \|(B^{-1}_{t} - B^{-1})\| ~\|C_{t}(w_{t})\|~|~\F_t] + \|B^{-1} \| ~\E[\| C_{t}(w_{t})  - C(w_{t}) \| ~|~\F_t].
	\end{align}
	First consider $\|C_{t}(w_{t}) \|$. We can bound
	\begin{align}
	\label{eq:ct-pf2}
	\| z_{\tau} \| &\leq \sum_{m=0}^{\tau} (\alpha \lambda)^{\tau-m} \|\phi(s_{m})\| \stackrel{(a)}{\leq} \bar{c}_{31} \\
	|\sigma_{\widehat{\U}_{s_{\tau}, \pi(s_{\tau})}} ( \Phi w_{t}) |  &=    | \inf_{u \in \widehat{\U}_{s_{\tau}, \pi(s_{\tau})}} u^{\top} ( \Phi w_{t})| \leq \sup_{u \in \widehat{\U}_{s_{\tau}, \pi(s_{\tau})}} |  u^{\top} ( \Phi w_{t})| \nonumber \\
	\label{eq:ct-pf3}
	& \leq \sup_{u \in \widehat{\U}_{s_{\tau}, \pi(s_{\tau})}}   \|  u\| ~ \|\Phi w_{t}\| \stackrel{(b)}{\leq} \bar{c}_{32}  \|w_{t}\|
	\end{align}
	where $\bar{c}_{31}$ and $\bar{c}_{32}$ are finite positive scalars. For $(a)$, we used the fact that $\sup_{s} \|\phi(s)\| < \infty$. For $(b)$, we used the fact that $ \sup_{u \in \widehat{\U}_{s_{\tau}, \pi(s_{\tau})}}   \|  u\|  < \infty$ since $\widehat{\U}_{s, \pi(s)}$ is a finite set for any $s\in\Ss$ and $\|\Phi w_{t}\| \leq c' \|w_{t}\|$ for some finite positive scalar $c'$. Using \eqref{eq:ct-pf2} and \eqref{eq:ct-pf3}, 
	\begin{align}
	\label{eq:ct-pf1}
	\|C_{t}(w_{t}) \| &\leq   \frac{\alpha}{(t+1)} \sum_{\tau=0}^{t} \| z_{\tau} \sigma_{\widehat{\U}_{s_{\tau}, \pi(s_{\tau})}} ( \Phi w_{t}) \|  \leq    \frac{\alpha}{(t+1)}  \sum_{\tau=0}^{t} \| z_{\tau} \| ~~|\sigma_{\widehat{\U}_{s_{\tau}, \pi(s_{\tau})}} ( \Phi w_{t}) | \leq   \bar{c}_{33} \|w_{t}\|.
	\end{align}
	From  Lemma 4.3 of \citep{nedic2003least}, we have $\E[ \|(B^{-1}_{t} - B^{-1})\|]  \leq \bar{c}_{34}/\sqrt{(t+1)}$. 
	Using this and \eqref{eq:ct-pf1}, we get, 
	\begin{align}
	\label{eq:ct-pf4}
	\E[ \|(B^{-1}_{t} - B^{-1})\| ~\|C_{t}(w_{t})\|] \leq \frac{ \bar{c}_{35}}{\sqrt{t+1}} ~ \|w_{t}\|.
	\end{align}
	For bounding $\E[\| C_{t}(w_{t})  - C(w_{t}) \|]$ in \eqref{eq:verify-c2-s1},   we define $V_t$ and $n_{t}$ as, 
	\begin{align*}
	V_t &= \frac{\alpha}{(t+1)} \sum_{m=0}^t \phi(s_{m}) \sum_{\tau=t+1}^\infty \left[ (\alpha \lambda {P^{o}_{\pi}})^{\tau-m} {\sigma}_{\widehat{\U}_{\pi}} (\Phi w_t) \right](s_{m}),~~ n_{t}(s) = \sum^{t}_{\tau = 0} \mathbb{I}\{s_{\tau} = s\}.
	\end{align*}
	Note that $n_{t}(s)$ is the number of visits to state $s$ until time $t$. Then,
	\begin{align}
	\E[C_t(w_t)|\F_t] &= \E \left[ \frac{\alpha}{(t+1)} \sum_{\tau=0}^t \sum_{m=0}^{\tau} (\alpha \lambda)^{\tau-m} \phi(s_{m}) \sigma_{\widehat{\U}_{s_{\tau}, \pi(s_{\tau})}} (\Phi w_t) | \F_t \right] \nonumber \\
	&= \E \left[ \frac{\alpha}{(t+1)}  \sum_{\tau=0}^t \sum_{m=0}^{\tau} (\alpha \lambda)^{\tau-m} \phi(s_{m}) \E \left[\sigma_{\widehat{\U}_{s_{\tau}, \pi(s_{\tau})}} (\Phi w_t) | \F_{m}\right] | \F_t \right] \nonumber \\
	&\stackrel{(a)}{=} \E \left[ \frac{\alpha}{(t+1)}  \sum_{\tau=0}^t \sum_{m=0}^{\tau} (\alpha \lambda)^{\tau-m} \phi(s_{m}) ((P^{o}_{\pi})^{\tau - m} \sigma_{\widehat{\U}_{\pi}} (\Phi w_t))(s_{m}) | \F_t \right]\nonumber  \\
	&\stackrel{(b)}{=} \E \left[ \frac{\alpha}{(t+1)}   \sum_{m=0}^{t} \phi(s_{m})  \sum_{\tau=m}^{t} ((\alpha \lambda P^{o}_{\pi})^{\tau - m} \sigma_{\widehat{\U}_{\pi}} (\Phi w_t))(s_{m}) | \F_t \right] \nonumber \\
	&= \E \left[ \frac{\alpha}{(t+1)}   \sum_{m=0}^{t} \phi(s_{m})  \sum_{\tau=m}^{\infty} ((\alpha \lambda P^{o}_{\pi})^{\tau - m} \sigma_{\widehat{\U}_{\pi}} (\Phi w_t))(s_{m}) | \F_t \right]  -  \E[V_{t} | \F_{t}] \nonumber\\
	\label{eq:ct_1}
	&\stackrel{(c)}{=} \E \left[ \frac{\alpha}{(t+1)}  \sum_{s \in \mathcal{S}} n_{t}(s) \phi(s)  \sum_{i=0}^{\infty} ((\alpha \lambda P^{o}_{\pi})^{i} \sigma_{\widehat{\U}_{\pi}} (\Phi w_t))(s) | \F_t \right]  -  \E[V_{t} | \F_{t}] ,
	\end{align} 
	where $(a)$ follows since the transition probability matrix governing along policy $\pi$ is $P^{o}_\pi$, $(b)$ by exchanging the order of summation, and $(c)$ by using the definition of $n_{t}(s)$.
	Note that,
	\begin{align}
	\E[C(w_t)|\F_t] &= \E[\alpha  \Phi^{\top} D  \sum^{\infty}_{i=0} (\alpha \lambda P^{o}_\pi)^{i}  {\sigma}_{\widehat{\U}_{\pi}} (\Phi w_t)|\F_t]  \nonumber\\&= \E [ \alpha  \sum_{s \in \mathcal{S}} d_{s} \phi(s)  \sum_{i=0}^{\infty} ((\alpha \lambda P^{o}_{\pi})^{i} \sigma_{\widehat{\U}_{\pi}} (\Phi w_t))(s) | \F_t ]. \label{eq:c_1}
	\end{align}
	So, using \eqref{eq:c_1} and \eqref{eq:ct_1},
	\begin{align}
	&\left\| \E\left[\left(C_t(w_t) - C(w_t)\right) |\F_t\right] \right\| \nonumber \\
	&= \left\| \alpha \sum_{s \in \mathcal{S}} \left(\frac{\E[n_{t}(s)]}{t+1} - d_{s}\right) \phi(s)  \sum_{i=0}^{\infty} ((\alpha \lambda P^{o}_{\pi})^{i} \sigma_{\widehat{\U}_{\pi}} (\Phi w_t))(s) - \E[V_t|\F_t] \right\|  \nonumber \\
	\label{eq:ct-pf5}
	&\stackrel{(d)}{\leq} \bar{c}_{36}  \sum_{s \in \mathcal{S}}  \left| \left(\frac{\E[n_{t}(s)]}{t+1} - d_{s}\right) \right|  \|w_{t}\| + \E[\|V_{t}\|~|\F_t]   \stackrel{(e)}{\leq}  \frac{\bar{c}_{37}}{{t+1}}  \|w_{t}\| +  \frac{\bar{c}_{38}}{{t+1}} \|w_{t}\|
	\end{align} 
	For getting (d), note that $\|\phi(s)  \sum_{i=0}^{\infty} ((\alpha \lambda P^{o}_{\pi})^{i} \sigma_{\widehat{\U}_{\pi}} (\Phi w_t))(s)\| \leq c' \| w_{t}\|$ for some positive constant $c'$, using \eqref{eq:ct-pf3} and the fact that the  summation is bounded due to the discounted factor $(\alpha \lambda)$. For getting $(e)$, we use the result from Lemma 4.2 \citep{nedic2003least} that $\left| \frac{\E[n_{t}(s)]}{t+1} - d_{s} \right|  \leq \frac{c'}{t+1}$ for some positive number $c'$. Also, it is straightforward to show that $\E[\|V_{t}\|~|\F_t] \leq c' \frac{\| w_{t}\|}{t+1}$ for some positive number $c'$ using \eqref{eq:ct-pf3} and the fact that the  summation is bounded due to the discounted factor $(\alpha \lambda)$.

	Using \eqref{eq:ct-pf4} and \eqref{eq:ct-pf5}  in \eqref{eq:verify-c2-s1}, we get
	\begin{align}
	\label{eq:verify-c2-s2}
	\|\E[\Delta^{3}_{t} (w_{t})] \| \leq \frac{\bar{c}_{3}}{\sqrt{t+1}} \|w_{t}\|.
	\end{align} 
	Notice that
	\begin{align}
	\label{eq:ct-pf6}
	\| w_{t} \| \leq \| w_{t} - w_{\pi} \| + \| w_{\pi}\| &\leq \|(\Phi^{\top} D \Phi)^{-1}\| ~ \|(\Phi^{\top} D \Phi) (w_{t} - w_{\pi}) \| + \| w_{\pi}\|  \nonumber \\
	&\leq \bar{c}_{39} (1 + \|\nabla f(w_{t}) \|).
	\end{align} 
	Now, 
	\begin{align}
	\|\E[e_{t} |  \F_{t}] \|  &= \|\E[\Delta^{1}_{t} ] \| ~ \|w_{t}\|  + \| \E[\Delta^{2}_{t} ] \| +  \| \E[\Delta^{3}_{t} (w_{t}) |  \F_{t}]  \| \nonumber \\
	\label{eq:verify-c2-s6}
	&\stackrel{(f)}{\leq} \frac{\bar{c}_{1}}{\sqrt{t+1}}   \|w_{t}\| + \frac{\bar{c}_{2}}{\sqrt{t+1}} + \frac{\bar{c}_{3}}{\sqrt{t+1}}  \|w_{t}\| \stackrel{(g)}{\leq}  \frac{c_{2}}{\sqrt{t+1}}  (1 + \|\nabla f(w_{t}) \|)
	\end{align}
	where $(f)$ is by using \eqref{eq:verify-c2-s4} and \eqref{eq:verify-c2-s2} and $(g)$ is  by using \eqref{eq:ct-pf6}. This completes the verification of the condition \eqref{eq:NB-condition2}.

	\textit{Verifying \eqref{eq:NB-condition3}:} 
	From definition, we have $h_t(w_t) + e_t = B_t^{-1} (A_t w_t + b_t + C_t(w_t)),$ for all $t.$ Using similar steps as before, it is straightforward to show that $\|B_{t}\| \leq c_{41}, |A_{t}| \leq c_{42}, |b_{t}| \leq c_{43}$ for some positive scalars $c_{41}, c_{42}, c_{43}$. From \eqref{eq:ct-pf1} we have $\|c_{t}(w_{t}) \| \leq \bar{c}_{33} \|w_{t}\|$. Combining these, we get
	\begin{align}
	    \|h_t(w_t) + e_t \| \leq c_{44} (1 + \| w_{t}\|).
	\end{align}
	Now, from \eqref{eq:ct-pf6} and \eqref{eq:NB-condition0}, $ \| h_t(w_t) + e_t  \|^2 \leq c_{45} (1+\| \nabla f( w_t)\|^2)$ for all $t .$ So, finally we have, \[ \E[\| h_t(w_t) + e_t  \|^2|\F_t] = \E[\| h_t(w_t) + e_t  \|^2|w_t]  \leq c_3 (1+\| \nabla f( w_t)\|^2) ~~\forall ~t , \] showing that \eqref{eq:NB-condition3} is satisfied.
	
	\textit{Verifying \eqref{eq:NB-condition4}:}  From \eqref{eq:verify-c2-s6}, $\epsilon_{t} = 1/\sqrt{(t+1)}$. So, this condition is satisfied. 
`
	So, all the assumption of Proposition \ref{thm:NB02-prop} are satisfied. Hence the result of that Proposition is true. In particular, $\nabla f(w_{t}) = \Phi^{\top} D \Phi (w_{t} - w_{\pi})$ converges to $0$. Since $\Phi^{\top} D \Phi$ is positive definite, this implies that $w_{t} \rightarrow w_{\pi}$.
\end{proof}



\section{Proof of the Results in Section \ref{sec:PI}} \label{sec:PI_appendix}
Let $\bar{V}_{k} = \Phi w_{\pi_{k}}$ and $V^{*}$ be the optimal robust value function. Define
\begin{align}
e_{k} = V_{\pi_{k}} - \bar{V}_{k},~~ l_k = V^* - V_{\pi_k},~~  g_k = V_{\pi_{k+1}} - V_{\pi_k}.
\end{align} Interpretations of these expressions: Since the robust value function in the $k^{\text{th}}$ iteration $\bar{V}_{k}$ is used as a surrogate for the robust value function $V_{\pi_{k}}$, $e_k$ quantifies the \emph{approximation error}. $g_k$ signifies the \emph{gain} of value functions between iterations $k$ and $k+1$. Finally, $l_k$ encapsulates the \emph{loss} in the value function because of using policy $\pi_k$ instead of the optimal policy.

Let $|x|$ denote element-wise absolute values of vector $x\in\R^{|\Ss|}$. We first prove the following result. This parallels to the  result for the nonrobust setting in  \citep{munos2003error}. 
\begin{lemma}\label{lem:l_g}
	We have
	\begin{align}
	|l_{k+1}| &\leq c(\alpha,\beta,0,\lambda)  \bar{H}_* (|l_k| + |e_k|) + c(\alpha,\beta,0,\lambda)  \bar{H}_k (|g_k| + |e_k|)~~\text{and} \label{eq:l_k}\\
	|g_k| &\leq  c(\alpha,\beta,0,\lambda) (I-c(\alpha,\beta,0,\lambda)  \bar{H}_{k+1})^{-1} ( \bar{H}_{k+1} +  \bar{H}_{k}) |e_k| \label{eq:g_k}
	\end{align}
	where the stochastic matrices $\bar{H}_*, \bar{H}_{k},$ and $\bar{H}_{k+1}$ are defined in \eqref{eq:H-matrices1}-\eqref{eq:H-matrices2}.
\end{lemma}
\begin{proof}
As before,  denote the operator $\widetilde{T}^{(\lambda)}_\pi $ as $\bar{T}^{(\lambda)}_\pi$ when  $\widehat{\U}_\pi = {\U}_\pi$. 	Now, similar to \eqref{eq:for_l_g_pi}, for any policy $\pi$ and $V_1,V_2$, we get that  	\begin{align}
	&|\bar{T}^{(\lambda)}_\pi (V_{1}) - \bar{T}^{(\lambda)}_\pi  (V_{2}) | \nonumber \\
	&= |(1-\lambda) \sum_{m=0}^\infty \lambda^{m} [\alpha \sum_{t=0}^{m} (\alpha P^{o}_{\pi})^t  ~ ({\sigma}_{{\U_\pi}} (V_{1}) - {\sigma}_{{\U_\pi}} (V_{2}) )+ (\alpha P^{o}_{\pi})^{m+1} (V_{1} - V_{2})] | \nonumber \\
	&\leq (1-\lambda) \sum_{m=0}^\infty \lambda^{m} [\alpha \sum_{t=0}^{m}   ~ |(\alpha P^{o}_{\pi})^{t}|{\sigma}_{{\U_\pi}} (V_{1}) - {\sigma}_{{\U_\pi}} (V_{2})| ~|+ |( \alpha P^{o}_{\pi})^{m+1} |V_{1} - V_{2}| | ] \nonumber \\
	&\stackrel{(a)}{\leq} (1-\lambda) \sum_{m=0}^\infty \lambda^{m} [\alpha \sum_{t=0}^{m}   ~ (\beta P^{o}_{\pi_{e}})^{t}|{\sigma}_{{\U_\pi}} (V_{1}) - {\sigma}_{{\U_\pi}} (V_{2})| + |( \beta P^{o}_{\pi_{e}})^{m+1} |V_{1} - V_{2}| | ] \nonumber \\
	&\stackrel{(b)}{\leq} (1-\lambda) \sum_{m=0}^\infty \lambda^{m} [\beta \sum_{t=0}^{m}   ~ (\beta P^{o}_{\pi_{e}})^{t}|V_{1} - V_{2}| + |( \beta P^{o}_{\pi_{e}})^{m+1} |V_{1} - V_{2}| | ] \label{eq:l_g_op}
	\end{align} where $(a)$ follows from Assumption \ref{as:offpolicy} and $\pi_{e}$ (dependent on $\pi$) being the exploration policy. $(b)$ follows from \eqref{eq:sigma_u_intermediate} in Lemma \ref{lemma_sigma_d_}.
	
	Recall that the optimal robust value function $V^{*}$ and the optimal robust policy $\pi^{*}$ satisfy the equation $\bar{T}^{(\lambda)}_{\pi^*} V^*  = V^{*}$. Using this, 
	\begin{align}
		l_{k+1} &= V^{*} - V_{\pi_{k+1}} = \bar{T}^{(\lambda)}_{\pi^*} V^* - \bar{T}^{(\lambda)}_{\pi_{k+1}} V_{\pi_{k+1}} \nonumber\\
		&= (\bar{T}^{(\lambda)}_{\pi^*} V^* - \bar{T}^{(\lambda)}_{\pi^*} V_{\pi_k} ) + (\bar{T}^{(\lambda)}_{\pi^*} V_{\pi_k}  - \bar{T}^{(\lambda)}_{\pi^*} \bar{V}_k )+ (\bar{T}^{(\lambda)}_{\pi^*} \bar{V}_k  - \bar{T}^{(\lambda)}_{\pi_{k+1}} \bar{V}_k) \nonumber \\
		&\hspace{2cm}+ (\bar{T}^{(\lambda)}_{\pi_{k+1}} \bar{V}_k - \bar{T}^{(\lambda)}_{\pi_{k+1}} V_{\pi_k} ) + (\bar{T}^{(\lambda)}_{\pi_{k+1}} V_{\pi_k}  - \bar{T}^{(\lambda)}_{\pi_{k+1}} V_{\pi_{k+1}} )\nonumber \\
		&\stackrel{(c)}{\leq} (\bar{T}^{(\lambda)}_{\pi^*} V^* - \bar{T}^{(\lambda)}_{\pi^*} V_{\pi_k} ) + (\bar{T}^{(\lambda)}_{\pi^*} V_{\pi_k}  - \bar{T}^{(\lambda)}_{\pi^*} \bar{V}_k )\nonumber \\
		&\hspace{2cm}+ (\bar{T}^{(\lambda)}_{\pi_{k+1}} \bar{V}_k - \bar{T}^{(\lambda)}_{\pi_{k+1}} V_{\pi_k} ) + (\bar{T}^{(\lambda)}_{\pi_{k+1}} V_{\pi_k}  - \bar{T}^{(\lambda)}_{\pi_{k+1}} V_{\pi_{k+1}} )\nonumber \\
	&\stackrel{(d)}{\leq} (1-\lambda) \sum_{m=0}^\infty \lambda^m [ \beta \sum_{t=0}^{m} (\beta P^o_{\pi^*})^t (|l_k| + |e_k|) + (\beta P^o_{\pi^*})^{m+1} (|l_k| + |e_k|)] \nonumber \\&\hspace{1cm}+ (1-\lambda) \sum_{m=0}^\infty \lambda^m [ \beta \sum_{t=0}^{m} (\beta P^o_{\pi_{k+1}})^t (|g_k| + |e_k|) + (\beta P^o_{\pi_{k+1}})^{m+1} (|g_k| + |e_k|)] \nonumber \\&\stackrel{(e)}{=} c(\alpha,\beta,0,\lambda)  \bar{H}_* (|l_k| + |e_k|) + c(\alpha,\beta,0,\lambda)  \bar{H}_{k+1} (|g_k| + |e_k|) \label{eq:l_k_1}
	\end{align} 
	Here $(c)$ follows because $\pi_{k+1}$ is the greedy policy w.r.t. $\bar{V}_k $ and hence $\bar{T}^{(\lambda)}_{\pi^*} \bar{V}_k  - \bar{T}^{(\lambda)}_{\pi_{k+1}} \bar{V}_k \leq 0$. $(d)$ follows from \eqref{eq:l_g_op} noting $(i)$ in Assumption \ref{as:policy_iteration_req}. Finally, $(e)$ follows by taking 
	\begin{align}\label{eq:H-matrices1}
	\bar{H}_* &= \frac{1}{c(\alpha,\beta,0,\lambda) } (1-\lambda) \sum_{m=0}^\infty \lambda^m [\beta\sum_{t=0}^{m} (\beta P^o_{\pi^*})^t  + (\beta P^o_{\pi^*})^{m+1}],~\text{and}\\\label{eq:H-matrices2}
	\bar{H}_{j} &= \frac{1}{c(\alpha,\beta,0,\lambda) } (1-\lambda) \sum_{m=0}^\infty \lambda^m [\beta\sum_{t=0}^{m} (\beta P^o_{\pi_{j}})^t  + (\beta P^o_{\pi_{j}})^{m+1}],~~\text{for}~j \geq 1.
	\end{align} 
	Note that, matrices $\bar{H}_*$ and $\bar{H}_{j}$ are stochastic matrices. This follows easily by verifying $\bar{H}_* \textbf{1} = \textbf{1} $, $\bar{H}_j \textbf{1} = \textbf{1} $  using algebra analysis as in \eqref{eq:contract_coeff_calc}. 
	
	The same argument can be repeated to get
	\begin{align*}
	&-l_{k+1} {\leq} c(\alpha,\beta,0,\lambda)  \bar{H}_* (|l_k| + |e_k|) + c(\alpha,\beta,0,\lambda)  \bar{H}_{k+1} (|g_k| + |e_k|).
	\end{align*} 
	Combining, we  get 
	\[ |l_{k+1}| \leq c(\alpha,\beta,0,\lambda)  \bar{H}_* (|l_k| + |e_k|) + c(\alpha,\beta,0,\lambda)  \bar{H}_{k+1} (|g_k| + |e_k|). \]
	
	Now, \begin{align*}
	&g_k = \bar{T}^{(\lambda)}_{\pi_{k+1}} V_{\pi_{k+1}} - \bar{T}^{(\lambda)}_{\pi_{k}}V_{\pi_k} \\
	&= \bar{T}^{(\lambda)}_{\pi_{k+1}} V_{\pi_{k+1}} - \bar{T}^{(\lambda)}_{\pi_{k+1}} V_{\pi_{k}} + \bar{T}^{(\lambda)}_{\pi_{k+1}}V_{\pi_k}  - \bar{T}^{(\lambda)}_{\pi_{k+1}} \bar{V}_{k} +( \bar{T}^{(\lambda)}_{\pi_{k+1}}\bar{V}_k - \bar{T}^{(\lambda)}_{\pi_{k}} \bar{V}_{k} )+ \bar{T}^{(\lambda)}_{\pi_{k}}\bar{V}_k  - \bar{T}^{(\lambda)}_{\pi_{k}}V_{\pi_k} \\
	&\geq \bar{T}^{(\lambda)}_{\pi_{k+1}} V_{\pi_{k+1}} - \bar{T}^{(\lambda)}_{\pi_{k+1}} V_{\pi_{k}} + \bar{T}^{(\lambda)}_{\pi_{k+1}}V_{\pi_k}  - \bar{T}^{(\lambda)}_{\pi_{k+1}} \bar{V}_{{k}} + \bar{T}^{(\lambda)}_{\pi_{k}}\bar{V}_k  - \bar{T}^{(\lambda)}_{\pi_{k}}V_{\pi_k}
	\end{align*} 
	where the last inequality follows because $\pi_{k+1}$ is the greedy policy w.r.t. $\bar{V}_k$ and hence $( \bar{T}^{(\lambda)}_{\pi_{k+1}}\bar{V}_k - \bar{T}^{(\lambda)}_{\pi_{k}} \bar{V}_{{k}} ) \geq 0$. From \eqref{eq:l_g_op}, noting $(i)$ in Assumption \ref{as:policy_iteration_req},  we have
	\begin{align*}
	-g_k &\leq(1-\lambda) \sum_{m=0}^\infty \lambda^m [ \beta \sum_{t=0}^{m} (\beta P^o_{\pi_{k+1}})^t (|g_k| + |e_k|) + (\beta P^o_{\pi_{k+1}})^{m+1} ( |g_k| + |e_k|)] \\
	&\hspace{1cm} + (1-\lambda) \sum_{m=0}^\infty \lambda^m [ \beta \sum_{t=0}^{m} (\beta P^o_{\pi_{k}})^t (|e_k|) + (\beta P^o_{\pi_{k}})^{m+1} ( |e_k|)] \\
	&{\leq} c(\alpha,\beta,0,\lambda)  \bar{H}_{k+1} ( |g_k| + |e_k|) + c(\alpha,\beta,0,\lambda)  \bar{H}_k ( |e_k|).
	\end{align*}
	Repeating the above argument for $-g_k =  \bar{T}^{(\lambda)}_{\pi_k} V_{\pi_k} - \bar{T}^{(\lambda)}_{\pi_{k+1}} V_{\pi_{k+1}}$, we get \[ |g_k| \leq  c(\alpha,\beta,0,\lambda)  \bar{H}_{k+1} ( |g_k| + |e_k|) + c(\alpha,\beta,0,\lambda)  \bar{H}_k ( |e_k|)  . \] So, $ |g_k| \leq  c(\alpha,\beta,0,\lambda) (I-c(\alpha,\beta,0,\lambda)  \bar{H}_{k+1})^{-1} ( \bar{H}_{k+1} +  \bar{H}_{k}) |e_k|   .$ Thus, proving the lemma.
\end{proof}

\begin{proof}[{\bf Proof of Theorem \ref{thm:RLSPI}}]
	From the Lemma \ref{lem:l_g}, taking $\limsup$ on both sides of \eqref{eq:l_k} we have \begin{align}
	&\limsup_{k\to\infty} |l_{k}| \leq \limsup_{k\to\infty}~c(\alpha,\beta,0,\lambda) (I-c(\alpha,\beta,0,\lambda)  \bar{H}_* )^{-1}( \bar{H}_*  +\bar{H}_k) |e_k| \nonumber \\&\hspace{3cm}+ c(\alpha,\beta,0,\lambda) (I-c(\alpha,\beta,0,\lambda)  \bar{H}_* )^{-1} \bar{H}_k |g_k| \nonumber \\
	&\stackrel{(a)}{\leq}  \limsup_{k\to\infty}~c(\alpha,\beta,0,\lambda) (I-c(\alpha,\beta,0,\lambda)  \bar{H}_* )^{-1}( \bar{H}_*  +\bar{H}_k) |e_k| \nonumber \\&\hspace{1cm}+ c^2(\alpha,\beta,0,\lambda) (I-c(\alpha,\beta,0,\lambda)  \bar{H}_* )^{-1} \bar{H}_k (I-c(\alpha,\beta,0,\lambda)  \bar{H}_{k+1})^{-1} ( \bar{H}_{k+1} +  \bar{H}_{k}) |e_k| \nonumber  \\
	&\stackrel{(b)}{=} \frac{2c(\alpha,\beta,0,\lambda) }{(1-c(\alpha,\beta,0,\lambda) )^2}  \limsup_{k\to\infty}~ H_k |e_k| \label{eq:l_e}
	\end{align} 
	where $(a)$ follows from \eqref{eq:g_k} in Lemma \ref{lem:l_g}. $(b)$ follows by taking
	\begin{align}
	H_k &= (1-c(\alpha,\beta,0,\lambda) )^2(I-c(\alpha,\beta,0,\lambda)  \bar{H}_* )^{-1} \Bigg( \frac{( \bar{H}_*  +\bar{H}_k)}{2} \Bigg. \nonumber \\
	&\hspace{3cm} \Bigg.+ c(\alpha,\beta,0,\lambda)  \bar{H}_k (I-c(\alpha,\beta,0,\lambda)  \bar{H}_{k+1})^{-1} \frac{( \bar{H}_{k+1} +  \bar{H}_{k})}{2} \Bigg).
	\end{align} 
	Notice that $H_k$ is a stochastic matrix. To see this, we know that, if $P_i,~i=\{1,2,3,4\}$ are stochastic matrices and $c<1$, then $P_1 P_2 P_3 P_4$, $(P_1+P_2)/2$, and $(1-c) (I- cP_1)^{-1}$ are valid stochastic matrices as well. Now, it is easy to verify that $H_k \textbf{1} = \textbf{1}$. Then, $\mu_k=\mu H_k$ is a valid probability distribution.
	
	Let $x^2$ denote element-wise squares of vector $x\in\R^{|\Ss|}$ and also let $\|x\|^2_\mu = \int_{s\in\Ss} |x(s)|^2 d\mu(s) = \mu |x|^2.$ Now, from \eqref{eq:l_e} we have \begin{align*}
	\limsup_{k\to\infty} \|l_{k}\|^2_\mu &= \limsup_{k\to\infty} \mu |l_{k}|^2 \leq  \frac{4c^2(\alpha,\beta,0,\lambda) }{(1-c(\alpha,\beta,0,\lambda) )^4} \limsup_{k\to\infty}~  \mu [H_k |e_k|]^2\\
	&\stackrel{(c)}{\leq}  \frac{4c^2(\alpha,\beta,0,\lambda) }{(1-c(\alpha,\beta,0,\lambda) )^4} \limsup_{k\to\infty}~  \mu H_k |e_k|^2\\
	&=  \frac{4c^2(\alpha,\beta,0,\lambda) }{(1-c(\alpha,\beta,0,\lambda) )^4} \limsup_{k\to\infty}~    \mu_k |e_k|^2 =  \frac{4c^2(\alpha,\beta,0,\lambda) }{(1-c(\alpha,\beta,0,\lambda) )^4}  \limsup_{k\to\infty}~ \| e_k\|^2_{\mu_k} \\
	&\stackrel{(d)}{\leq}  \frac{4C_1 C_2 c^2(\alpha,\beta,0,\lambda) }{(1-c(\alpha,\beta,0,\lambda) )^4} \limsup_{k\to\infty}~ \|e_k\|^2_{d_{\pi_k}}.
	\end{align*} 
	Here $(c)$ follows from Jensen's inequality. To see this, let \begin{equation*}
	H_k = 
	\begin{pmatrix}
	- q_1 - \\
	- q_2 - \\
	\vdots \\
	- q_{|\Ss|} -
	\end{pmatrix} ~\text{where }q_i \in \R^{|\Ss|},~\text{for all } i,~\text{are probability vectors.} 
	\end{equation*} For any $x\in\R^{|\Ss|}$, let $x(j)$ denote the $j^{\text{th}}$ coordinate value in $x$. Now, for each $i\in\{1,2,\ldots,|\Ss|\}$, define $|\Ss|$-discrete valued random variable $X_i$ such that it takes value $|e_k|(j)$ with probability $q_i(j)$ for all $j\in\{1,2,\ldots,|\Ss|\}$. Thus, from Jensen's inequality, we have \begin{align}
	     [H_k |e_k|]^2 = ((\E[X_1])^2, (\E[X_2])^2, \cdots, (\E[X_{\Ss}])^2 )^\top
	 \leq (\E[X_1^2], \E[X_2^2], \cdots, \E[X_{\Ss}^2])^\top = H_k |e_k|^2.  \label{eq:stoc_matx_jensen}
	\end{align}
	$(d)$ follows by noting that for any $x\in\R^{|\Ss|}$, from $(iv)$ in Assumption \ref{as:policy_iteration_req}, we have $\|x\|_{\mu_k}^2 \leq C_1\|x\|_{\bar{\mu}}^2 \leq C_1 C_2\|x\|_{d_{\pi_k}}^2 $ for all $k$. Thus proving \eqref{eq:pi_thm_1} of this theorem.
	
	Now, since $\rho=0$, \eqref{eq:pi_thm_2} of this theorem directly follows from $(iii)$ in Assumption \ref{as:policy_iteration_req} and \eqref{eq:Ttilde-contraction-3} in Theorem \ref{prop:approx-td-op-1}. This completes the proof of this theorem.
\end{proof}

Now we make an alternative assumption to Assumption \ref{as:policy_iteration_req}.$(iv)$ to get a guarantee for any $K^{\text{th}}$ iteration of the RLSPI algorithm.

\begin{assumption}\label{as:finite_time_pi}
For an arbitrary sequence of stationary policies $\{\pi_i\}_{i\geq 1}$, let some probability distributions $\mu$ and $\bar{\mu}$ satisfy
\begin{align}
    {C_{\mu,\bar{\mu}}} = (1-c(\alpha,\beta,0,\lambda)) \sum_{m\geq 1} (c(\alpha,\beta,0,\lambda))^m c_{\mu,\bar{\mu}}(m) < \infty, \label{eq:C_mu}
\end{align} where for any given $m\geq 1$ the coefficients are defined as \begin{align}
    c_{\mu,\bar{\mu}}(m) = \sup_{\pi_1,\ldots,\pi_m} \left\| \frac{\dd(\mu \bar{H}_1\bar{H}_2\ldots\bar{H}_m)}{\dd\bar{\mu}} \right\|_\infty, \label{eq:c_mu_coefficients}
\end{align}
and $\bar{H}_k$ is a stochastic matrix that depends on $P^o_{\pi_{k}}$ for any $1\leq k \leq m$. Also assume that $d_{\pi_{k}} \geq \bar{\mu}/C$ for all $k$. 
\end{assumption}

We note that $c_{\mu,\bar{\mu}}(m)$ can potentially diverge to $\infty$, but $C_{\mu,\bar{\mu}}$ is finite as long as $(c(\alpha,\beta,0,\lambda))^m$ converges to $0$ at a faster rate. Assumption \ref{as:finite_time_pi} being similar as in the non-robust setting, we refer the reader to \citep{antos2008learning,lazaric2012finite} for its detailed interpretation.

Here is a result that provides a guarantee for the performance of the policy learned in $K^{\text{th}}$ iteration of the RLSPI algorithm.
\begin{theorem}
\label{thm:RLSPI_for_any_K}
Let Assumption \ref{as:offpolicy}, Assumption \ref{as:policy_iteration_req}.(i)-(iii), and Assumption \ref{as:finite_time_pi} hold. Let the range of reward function $r$ be $(0,R_{\max}]$. Let $\{\pi_{k}\}_{k=1}^K$  be the sequence of policies generated by the RLSPI algorithm for some $K\geq 1$. Let $V_{\pi_{k}}$ and $\bar{V}_{k} = \Phi w_{\pi_{k}}$ be true robust value function and the approximate robust value function corresponding to the policy $\pi_{k}$. Also, let  $V^*$ be the optimal robust value function. Then, with $c(\alpha,\beta,0,\lambda)<1$,
\begin{align} \label{eq:pi_thm_2_1}
   \| V^* - V_{\pi_K} \|_\mu\leq \frac{2\sqrt{2} (c(\alpha,\beta,0,\lambda))^{(K+1)/2} }{(1-c(\alpha,\beta,0,\lambda) )^{3/2}}~ R_{\max} + \frac{2\sqrt{2} c(\alpha,\beta,0,\lambda) \sqrt{C C_{\mu,\bar{\mu}}} }{(1-c(\alpha,\beta,0,\lambda) )^2} \max_{0\leq k < K} \|V_{\pi_{k}} - \bar{V}_{k}\|_{d_{\pi_k}}.
\end{align}
Moreover, from Theorem \ref{prop:approx-td-op-1} and  Assumption \ref{as:policy_iteration_req}.$(iii)$, as $K\to\infty$  we have
\begin{align} \label{eq:pi_thm_2_2}
    \| V^* - V_{\pi_K} \|_\mu\leq  \frac{2\sqrt{2} c(\alpha,\beta,0,\lambda) \sqrt{C C_{\mu,\bar{\mu}}} }{(1-c(\alpha,\beta,0,\lambda) )^3} ~ \delta  .
\end{align}
\end{theorem}
\begin{proof}[{\bf Proof of Theorem \ref{thm:RLSPI_for_any_K}}]
	Let $\zeta=c(\alpha,\beta,0,\lambda)$. From the Lemma \ref{lem:l_g}, using \eqref{eq:g_k} in \eqref{eq:l_k} we have \begin{align}
	|l_{k+1}| &\leq \zeta \bar{H}_* |l_{k}| +  \zeta ( \bar{H}_*  +\bar{H}_k) |e_k| + \zeta  \bar{H}_k |g_k| \nonumber \\
	&\leq \zeta \bar{H}_* |l_{k}| +  \zeta ( \bar{H}_*  +\bar{H}_k) |e_k| + \zeta^2  \bar{H}_k (I-\zeta  \bar{H}_{k+1})^{-1} ( \bar{H}_{k+1} +  \bar{H}_{k}) |e_k| \nonumber \\
	&\stackrel{(a)}{=} \zeta \bar{H}_* |l_{k}| + \frac{2\zeta}{1-\zeta} H_k |e_k|, \label{eq:l_e_1}
	\end{align} 
	where $(a)$ follows since
	\begin{align}
	H_k &= (1-\zeta ) \Bigg( \frac{( \bar{H}_*  +\bar{H}_k)}{2} + \zeta  \bar{H}_k (I-\zeta  \bar{H}_{k+1})^{-1} \frac{( \bar{H}_{k+1} +  \bar{H}_{k})}{2} \Bigg).
	\end{align} 
	
	Taking $(K-1)$-recursions of \eqref{eq:l_e_1} we get \begin{align} |l_K| \leq (\zeta \bar{H}_*)^K |l_0| + \frac{2\zeta}{1-\zeta} \sum_{k=0}^{K-1} (\zeta \bar{H}_*)^{K-k-1} H_k |e_k|. \label{eq:l_recur} \end{align}  Note that since $\alpha\leq\beta$, we have that $\alpha\leq\zeta$. Now, since the rewards are in $(0,R_{\max}]$, we also have that $|l_0| = |V^* - V_{\pi_0}| \leq \frac{2 R_{\max}}{(1-\alpha)} \mathbbm{1} \leq \frac{2 R_{\max}}{(1-\zeta)} \mathbbm{1}.$ Thus, bounding \eqref{eq:l_recur} further we get \begin{align} |l_K| \leq \frac{2\zeta_K \zeta}{(1-\zeta)^2} \left[\alpha_K G_K R_{\max} \mathbbm{1} + \sum_{k=0}^{K-1} \alpha_k G_k |e_k|\right], \label{eq:l_K_final} \end{align} where $\zeta_K = \zeta^{K-1} (2-\zeta-\zeta^K) \leq 2$, the positive coefficients $\alpha$s are \begin{align}
	&\alpha_k = \frac{\zeta^{K-k-1}(1-\zeta)}{\zeta_K}, \text{ for } 0\leq k \leq K-1, \text{ and } \alpha_K = \frac{\zeta^{K-1}(1-\zeta)}{\zeta_K}, \label{eq:alpha}
	\end{align} and the operators $G$s are \[ G_k = (\bar{H}_*)^{K-k-1} H_k, \text{ for } 0\leq k \leq K-1, \text{ and } G_K = (\bar{H}_*)^K. \] Note that $\sum_{k=0}^K \alpha_k = 1$ and $G_k$ for $0\leq k \leq K$ are stochastic matrices.
	
	Now, from \eqref{eq:l_K_final} we have \begin{align*}
    \|l_{K}\|^2_\mu &\stackrel{(b)}{\leq}  \frac{4 \zeta_K^2 \zeta^2 }{(1-\zeta )^4} ~\mu \left[\alpha_K (G_K R_{\max} \mathbbm{1})^2 + \sum_{k=0}^{K-1} \alpha_k (G_k |e_k|)^2\right]\\
	&\stackrel{(c)}{\leq}  \frac{4 \zeta_K^2 \zeta^2 }{(1-\zeta )^4}~ \mu \left[\alpha_K G_K R^2_{\max} \mathbbm{1} + \sum_{k=0}^{K-1} \alpha_k G_k |e_k|^2\right]\\
	&=  \frac{4 \zeta_K^2 \zeta^2 }{(1-\zeta )^4}~ \left[\alpha_K R^2_{\max} + \mu \sum_{k=0}^{K-1} \alpha_k G_k |e_k|^2\right]\\
	&\stackrel{(d)}{\leq}  \frac{4 \zeta_K^2 \zeta^2 }{(1-\zeta )^4}~ \left[\alpha_K R^2_{\max} + (1-\zeta) \sum_{k=0}^{K-1} \alpha_k \sum_{m\geq 0} \zeta^m c_{\mu,\bar{\mu}}(m+K-k) \|e_k\|^2_{\bar{\mu}}\right]\\
	&\stackrel{(e)}{\leq}  \frac{4 \zeta_K^2 \zeta^2 }{(1-\zeta )^4}~ \left[\alpha_K R^2_{\max} + (1-\zeta) \|e\|^2_{\bar{\mu}} \sum_{k=0}^{K-1} \alpha_k \sum_{m\geq 0} \zeta^m c_{\mu,\bar{\mu}}(m+K-k) \right]\\
	&\stackrel{(f)}{=} \frac{4 \zeta_K \zeta^2 }{(1-\zeta )^4}~ \left[{\zeta^{K-1}(1-\zeta)R^2_{\max}}  + (1-\zeta) \|e\|^2_{\bar{\mu}} \sum_{k=0}^{K-1} \zeta^{K-k-1}(1-\zeta) \sum_{m\geq 0} \zeta^m c_{\mu,\bar{\mu}}(m+K-k) \right] \\
	&\stackrel{(g)}{\leq} \frac{8 \zeta^2 }{(1-\zeta )^4}~ \left[\zeta^{K-1}(1-\zeta)R^2_{\max}  +  \|e\|^2_{\bar{\mu}} C_{\mu,\bar{\mu}}\right] \\
	&= \frac{8 \zeta^{K+1} }{(1-\zeta )^3}~ R^2_{\max}  + \frac{8 \zeta^2 }{(1-\zeta )^4} \|e\|^2_{\bar{\mu}} C_{\mu,\bar{\mu}} \stackrel{(h)}{\leq} \frac{8 \zeta^{K+1} }{(1-\zeta )^3}~ R^2_{\max}  + \frac{8 \zeta^2 C C_{\mu,\bar{\mu}} }{(1-\zeta )^4} \|e\|^2_{d_{\pi_k}} ,
	\end{align*} where $(b)$ follows from Jensen's inequality, i.e., $f(\sum_{k=0}^K \alpha_k x_k) \leq \sum_{k=0}^K \alpha_k f(x_k) $ for any convex function $f$. Also, $(c)$ follows from Jensen's inequality, similar to \eqref{eq:stoc_matx_jensen}, with stochastic matrix $G_k$. $(d)$ follows from the definition of the coefficients $c_{\mu,\bar{\mu}}(m)$ \eqref{eq:c_mu_coefficients}, $(e)$ follows by taking $e$ such that $\|e\|^2_{\bar{\mu}}=\max_{0\leq k\leq K-1}\|e_k\|^2_{\bar{\mu}}$, $(f)$ follows from \eqref{eq:alpha}, $(g)$ follows since $\zeta_K\leq 2$ and the definition of $C_{\mu,\bar{\mu}}$ \eqref{eq:C_mu}, and $(h)$ follows by noting that for any $x\in\R^{|\Ss|}$, from Assumption \ref{as:finite_time_pi}, we have $\|x\|_{\bar{\mu}}^2 \leq C \|x\|_{d_{\pi_k}}^2 $ for all $k$. Using the fact that $a^2 + b^2 \leq (a+b)^2$ for $a\geq 0, b\geq 0$ completes the proof of \eqref{eq:pi_thm_2_1} in this theorem.
	
	Now, since $\rho=0$, \eqref{eq:pi_thm_2_2} of this theorem directly follows from $(iii)$ in Assumption \ref{as:policy_iteration_req} and \eqref{eq:Ttilde-contraction-3} in Theorem \ref{prop:approx-td-op-1}. This completes the proof of this theorem.
\end{proof}

\section{Experiments}
\label{sec:experiments}

In all the experiments reported, we use a spherical uncertainty set $\{x:\|x\|_2 \leq r\}$ where $r$ is the radius parameter.  For such a set, we can compute a closed form solution for $\sigma_{\widehat{\U}_{s,\pi(s)}}(\Phi w)$ as $\sqrt{r w^\top \Phi^\top \Phi w}$ \citep{roy2017reinforcement}. Note that, we can precompute $ \Phi^\top \Phi$ once and reuse it in every iteration of the RLSPI Algorithm, thus saving the computational overhead.

\paragraph{Chain MDP:} We first consider a tabular MDP problem represented in the Figure \ref{fig:simplechain} for verifying the convergence of RLSPI algorithm. This MDP consists of $10$ states depicted by circles here. We have two actions, that is, move left or right. The actions fail to remain in a given direction with probability $0.1$, depicted by the red arrows. Thus, with probability $0.9$, actions succeed to be in a given direction, depicted by the blue (action left being unchanged) and green (action right being unchanged) arrows. Finally, visiting states of yellow color, that is $0$ and $9$, are rewarded $1$, and visiting other states are rewarded $0$.

\citep{lagoudakis2003least} observes that  learning algorithms often attain sub-optimal policies under such MDPs due to the randomization of actions (as depicted by red arrows in the Figure \ref{fig:simplechain}). It is also straightforward that the optimal policy of this MDP is moving left for states 0 through 4 and moving right for states 5 through 9. We train RLSPI algorithm  on this MDP with  $\alpha = 0.9, \lambda = 0$. We use the space spanned by polynomials, degree up to $2$, as the feature space and set $\delta = 0.1$ (error of weights as mentioned in Step 8 of the RLSPI algorithm). We select $r$ as $0.01$ times the constant $\| \Phi^\top \Phi \|_F^{-1}$ where $\| \cdot \|_F$ is the Frobenius norm. 
\begin{figure}[ht]
	\centering
	\begin{minipage}{.49\textwidth}
		\centering
		\includegraphics[width=\linewidth]{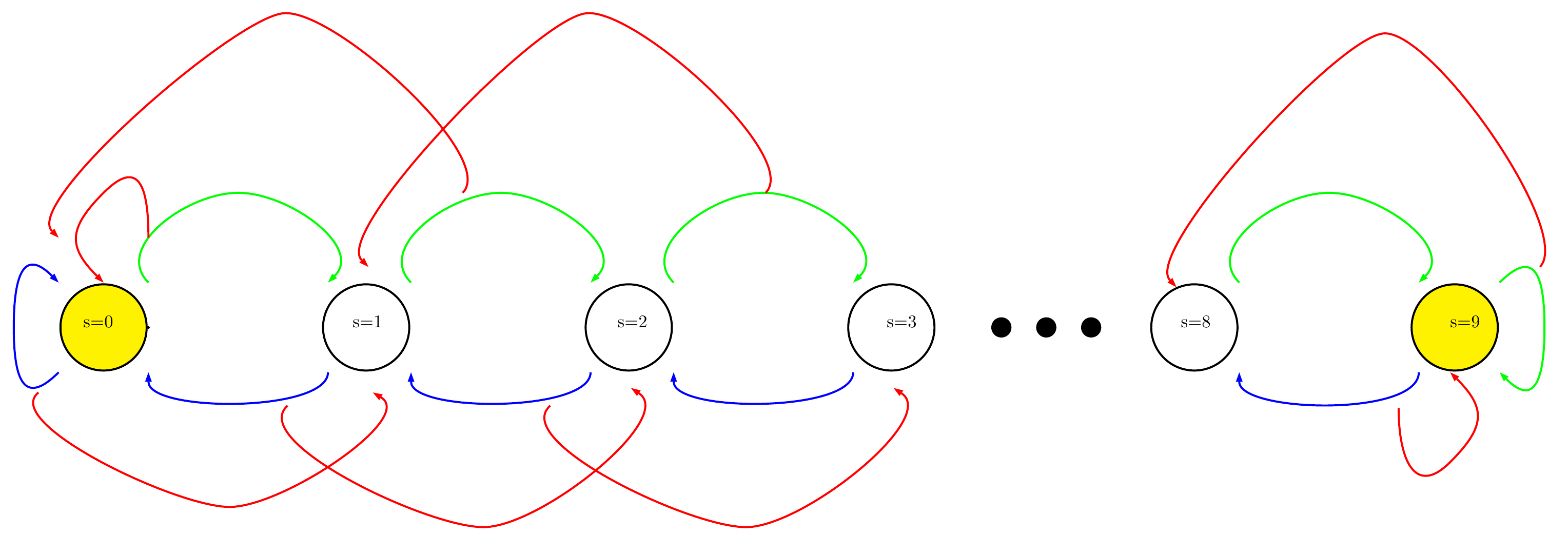}
		\captionof{figure}{}
		\label{fig:simplechain}
	\end{minipage}
	\begin{minipage}{.49\textwidth}
		\centering
		\includegraphics[width=\linewidth]{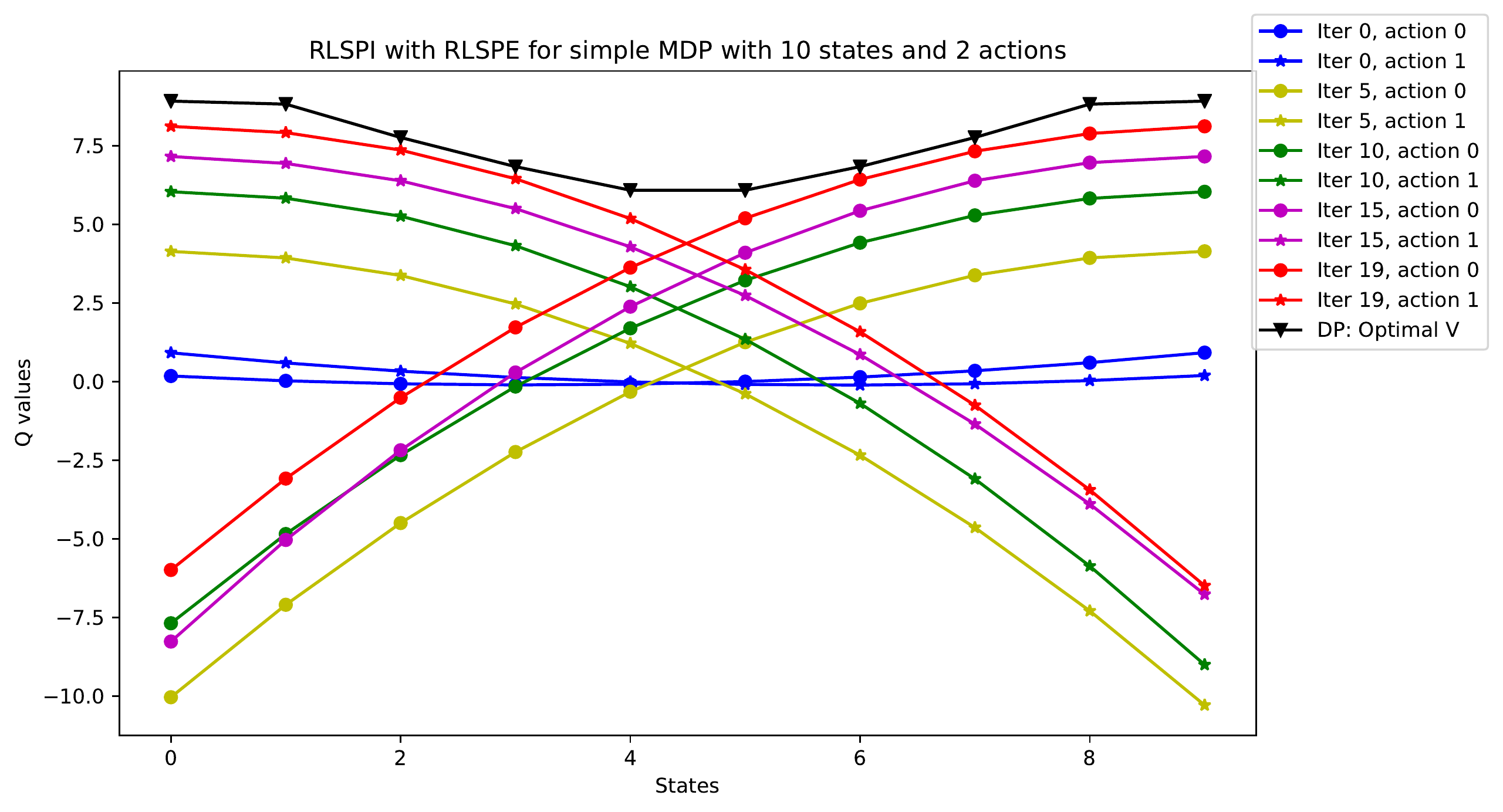}
		\captionof{figure}{}
		\label{fig:sc_rlspi_lspe}
	\end{minipage}
\end{figure}

Figure \ref{fig:sc_rlspi_lspe} shows how the Q-value functions in RLSPI algorithm  training evolve as the iteration progress. From this, we note that RLSPI algorithm is able to find the optimal policy with relatively less number of iterations. From this figure, we also note that the Q-value functions corresponding to the optimal policy in RLSPI algorithm converges to the optimal robust value function. 

\paragraph{Examples from OpenAI Gym \citep{brockman2016openai}:} We now provide more details for the OpenAI Gym experiments demonstrated in Section \ref{sec:simulations}. We use the radial basis functions (RBFs) for the purpose of feature spaces in our experiments. The general expression for RBFs is $\psi(x) = \exp(-\|x-\mu\|^2/\sigma)$ where the RBF parameters $\mu$ and $\sigma$ are chosen before running the experiment. Here $x$ is a concatenation of states and actions when both $\Ss,\Aa$ are continuous spaces. In this case, the feature map is simply defined as $\phi(s,a)=\psi((s,a))$ where $(s,a)$ represents the concatenation operation. \begin{wrapfigure}{r}{9cm}
	\includegraphics[width=9cm]{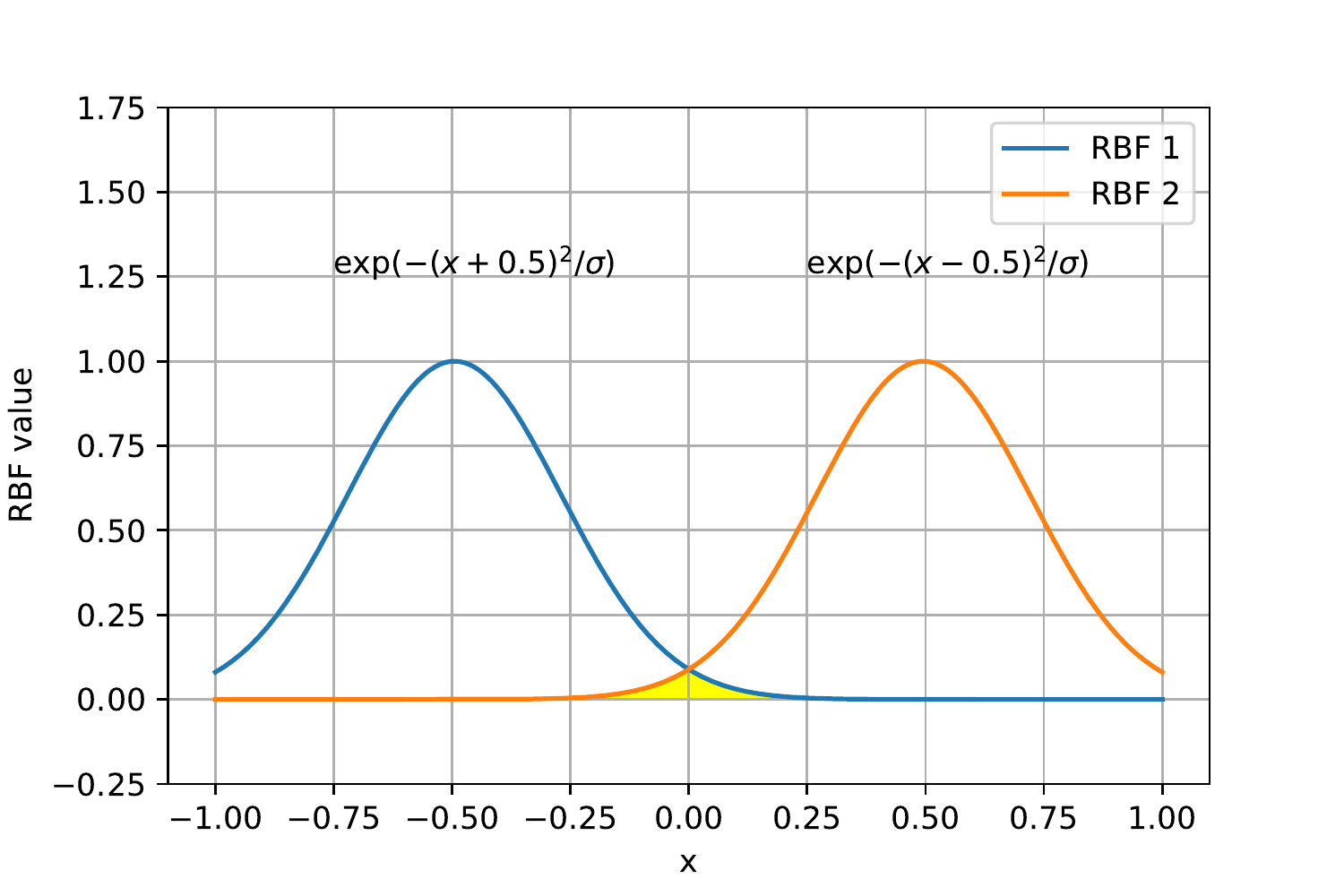}
	\caption{We have discretized $x$ with $10^4$ points here. The overlap percentage (ratio of the area in the yellow region to the area of either RBF 1 or RBF 2) is $2.88$.}
	\label{fig:rbf}
\end{wrapfigure} While working with experiments whose action space is discrete, we naturally choose $\phi(s,a)$ to be the vector $(\mathbbm{1}(a=1)\psi(s),\ldots,\mathbbm{1}(a=|\Aa|)\psi(s))^\top$ where $\mathbbm{1}(E)$ is the indicator function which produces value $1$ if the event $E$ is true, and $0$ otherwise. After a few trials, we observed that using few (typically $3$-$4$) uniformly spaced RBFs in each dimension of $x$, here $x$ is as described before, with approximate overlap percentage of $2.5$ works suitably for getting the desired results shown here and in Section \ref{sec:simulations}. Figure \ref{fig:rbf} illustrates this for the case of using two$(=n)$ uniformly spaced RBFs with one dimensional $x$ variable. For this illustration, we have the low$(l)$ and high$(h)$ values of $x$ to be $-0.5$ and $0.5$ respectively. Thus, the centers (i.e., parameter $\mu$) of the two uniformly spaced RBFs in Figure \ref{fig:rbf} are $-0.5$ and $0.5$. We select the parameter $\sigma$ as $(h-l)^2/n^3=0.125$. We execute this idea on the OpenAI Gym environments. We also experiment on FrozenLake8x8 OpenAI Gym environment, for which we use the tabular feature space since the state space is discrete.

A short description of the  OpenAI Gym tasks {\bf CartPole, MountainCar, Acrobot, and FrozenLake8x8} we used are as follows.

\textbf{\emph{CartPole:}} By a hinge, a pole is attached to a cart, which moves along a one-dimensional path. The motion of the cart is controllable, which is either to move it left or right. The pole starts upright, and the goal is to prevent it from falling over. A reward of $+1$ is provided for every time-step that the pole remains upright. CartPole consists of a $4$-dimension continuous state space with $2$ discrete actions. \\ 
\textbf{\emph{MountainCar:}} A car is placed in the valley and there exists a flag on top of the hill. The goal is to reach the flag. The control signal is the acceleration and deceleration in continuous domain. A reward of $0$ is provided if the car reaches goal, otherwise it is provided $-1$. MountainCar consists of a $3$-dimension $\Ss\times \Aa$ continuous space. \\
\textbf{\emph{Acrobot:}} Two poles attached to each other by a free moving joint and one of the poles is attached to a hinge on a wall. Initially, the poles are hanging downwards. An action, positive and negative torques can be applied to the movable joint. A reward of $-1$ is provided every time-step until the end of the lower pole reaches a given height, at termination $0$ reward is provided. The goal is to maximize the reward gathered. Acrobot consists of a $6$-dimension continuous state space with $3$ discrete actions. \\ 
\textbf{\emph{FrozenLake8x8:}} A grid of size $8\times 8$ consists of some tiles which lead to the agent falling into the water. The agent is rewarded $1$ after reaching a goal tile without falling and rewarded $0$ in every other timestep.

\begin{figure}[ht]
	\centering
	\begin{minipage}{.32\textwidth}
		\centering
		\includegraphics[width=\linewidth]{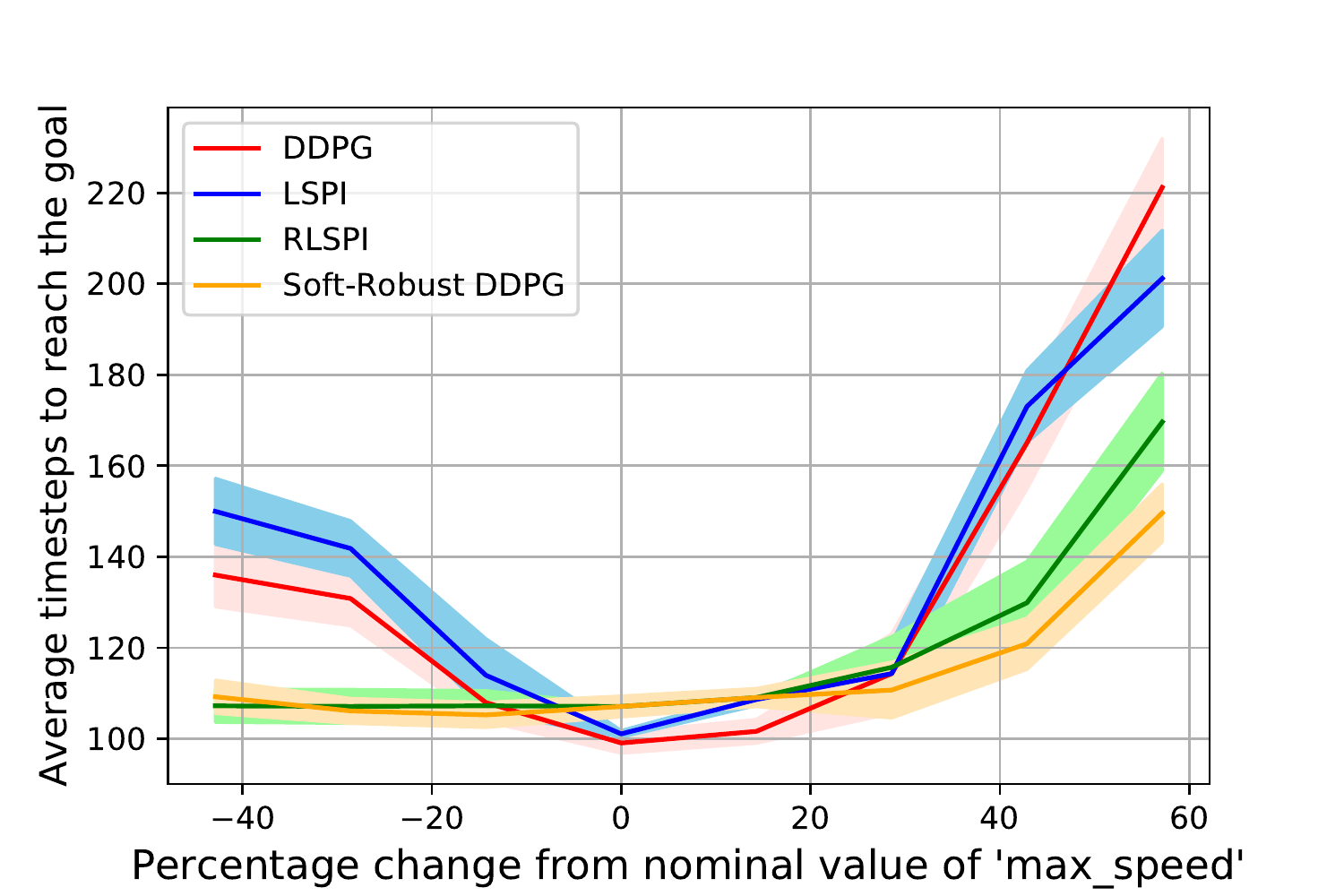}
		\captionof{figure}{}
		\label{fig:mt_speed}
	\end{minipage}
	\begin{minipage}{.32\textwidth}
		\centering
		\includegraphics[width=\linewidth]{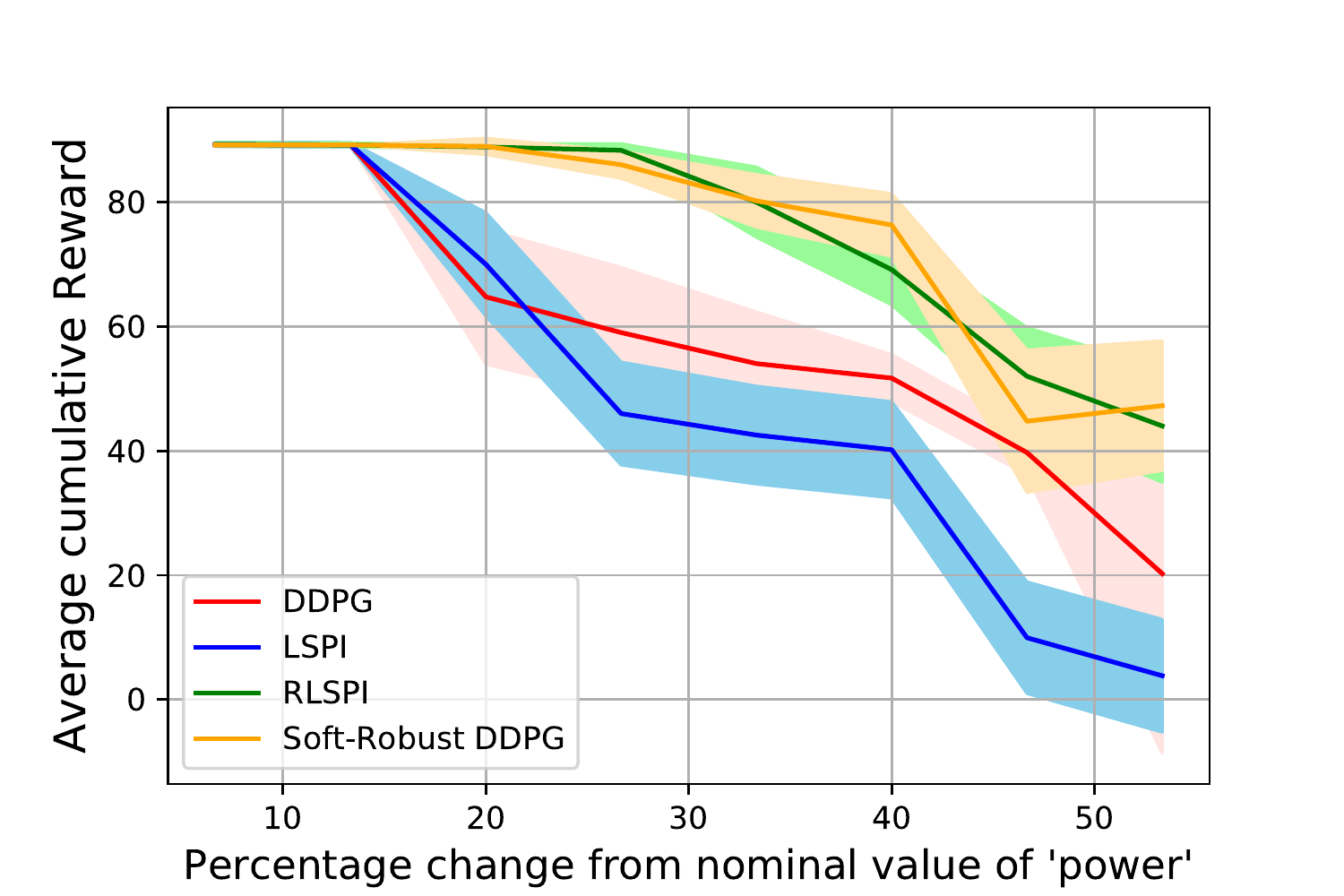}
		\captionof{figure}{}
		\label{fig:mt_power}
	\end{minipage}
	\begin{minipage}{.32\textwidth}
		\centering
		\includegraphics[width=\linewidth]{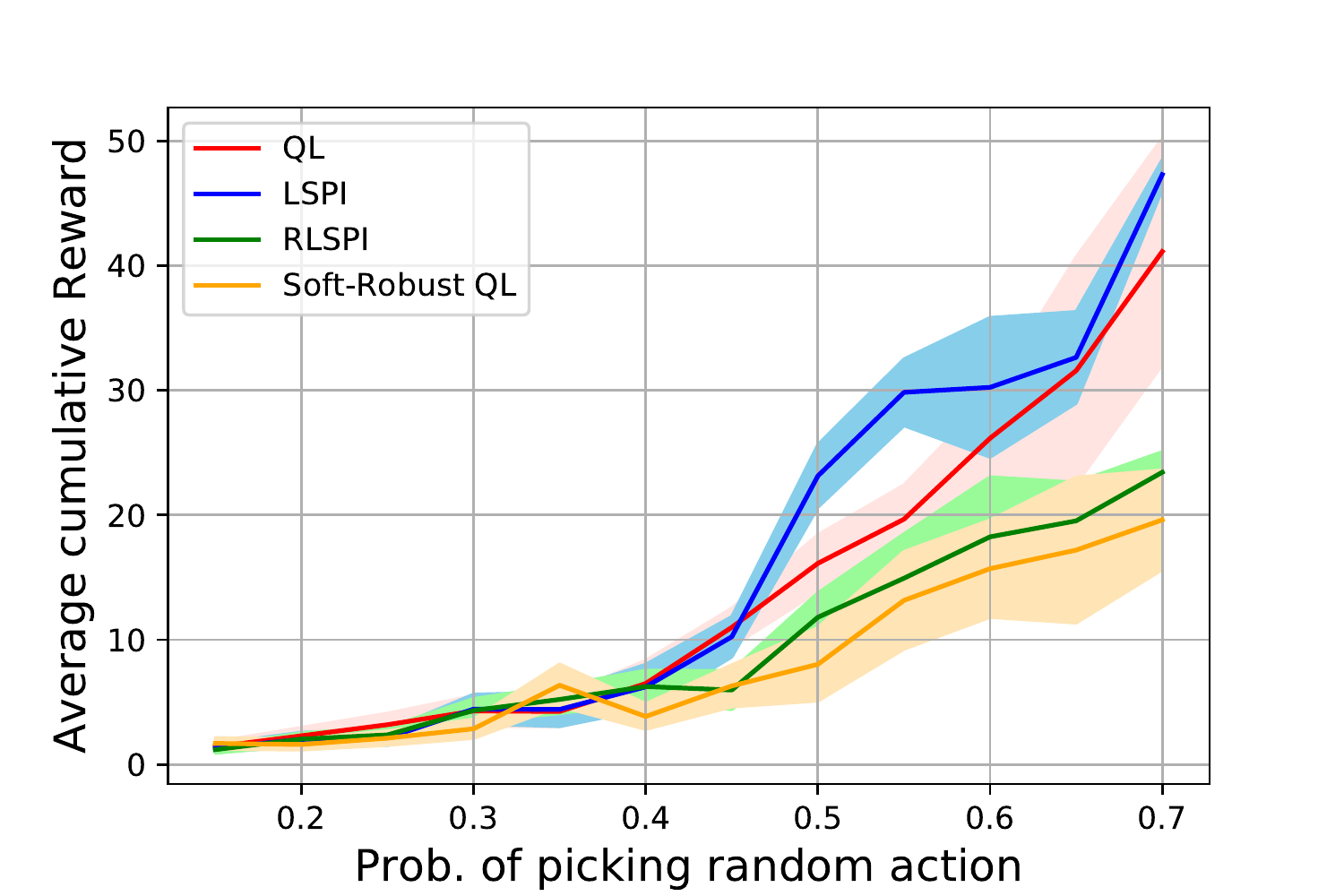}
		\captionof{figure}{}
		\label{fig:frolake_failprob}
	\end{minipage}
\end{figure}

In Section \ref{sec:simulations}, we provided performance evaluation curves in Figures \ref{fig:cart_failprob}-\ref{fig:cart_pole-in}. Here, we provide more results.

Figures \ref{fig:mt_speed} shows the average time steps to reach the goal in MountainCar environment as we change the parameter \emph{max\_speed}. The nominal value of this parameter is $0.07$. As the parameter deviates from the nominal value, the performance of the policy obtained by the LSPI algorithm degrades quickly whereas the performance of the policy obtained by the RLSPI algorithm is fairly robust. Figure \ref{fig:mt_power} shows the average cumulative reward on the MountainCar environment as we change the parameter \emph{power}. The nominal value of this parameter is $15\times10^{-4}$. We again note that the RLSPI algorithm showcases robust performance. Figure \ref{fig:frolake_failprob} shows the ratio of average time to reach the goal and the number of trajectories which actually reach the goal on the FrozenLake8x8 environment against probability of picking a random action. Note that for large values of this probability all algorithms take more time to reach the goal or often fall into the water. Here again, RLSPI shows robust performance. Intuitively, perturbation in the parameters (like the action space, CartPole's \emph{force\_mag, gravity, length}, MountainCar's \emph{max\_speed, power}) of the environment is captured by the uncertainty set in the RMDP framework. Thus we see good performances of the robust algorithms like our RLSPI algorithm, Soft-Robust algorithms \cite{derman2018soft}, and Robust Q-learning algorithm \cite{roy2017reinforcement} compared to the non-robust algorithms.

In each policy iteration loop, in both LSPI and RLSPI algorithms, we generate $t$ trajectories of horizon length $h$ using the last updated policy (the initial policy $\pi_0$ is random.) We generally stop the simulation after $10$-$20$ policy iteration loops. The details of the hyper-parameters are shown in Table \ref{tab:hp} in addition to $\lambda$ being set to zero. 
\begin{table}[h]
	\begin{center}
		\begin{tabular}{|c|c|c|c|c|}
			\hline 
			OpenAI Gym & Discount & Weights error &  & \\ 
			Environment & $\alpha$ & $\epsilon_0$ & $t$ & $h$ \\ 
			\hline & &&& \vspace{-0.3cm}\\
			CartPole & $0.95$ & $0.01$ & $150$ & $200$\\
			MountainCar & $0.95$ & $0.05$ & $1000$ & $20$\\
			Acrobot & $0.98$ & $0.1$ & $100$ & $200$\\
			FrozenLake8x8 & $0.99$ & $0.01$ & $3000$ & $200$\\ \hline
		\end{tabular}
	\end{center}\caption{Details of hyper-parameters in LSPI and RLSPI algorithms experiments.} \label{tab:hp}
\end{table}

Here are the details on the Q-learning based algorithms. The Q-learning algorithm with linear function approximation on CartPole uses $64$ parameterized (centers and variances) RBFs chosen by the Adam optimizer. Both Q-learning and Soft-Robust Q-learning algorithms for FrozenLake8x8 uses the tabular method, instead of deep neural or linear function architectures. We use the usual decaying-epsilon-greedy for the exploration policies, such that it exponentially decays to $0.01$ at half-way through the total number of training episodes (Epochs in Table \ref{tab:hp2}) starting from $0.99$. We provide the hyper-parameters in Table \ref{tab:hp2} like the discount factor, size of hidden layers $h$ starting from the first hidden layer in the given array, and size of the batch of tuples (state, action, next state, reward) chosen uniformly from the experience buffer of size $5000$ to update the neural network (Batch in Table \ref{tab:hp2}). For all neural networks, we used the \emph{relu} activation functions. Note that the DDPG algorithm uses two same sized neural networks for actor and critic.
\begin{table}[ht]
	\begin{center}
		\begin{tabular}{|c|c|c|c|c|}
			\hline 
			OpenAI Gym & Discount & Hidden layers  & Batch & Epochs \\ 
			Environment & $\alpha$ & $h$ & & \\ \hline
			CartPole & $0.9$ & $[200,100]$ & $200$ & $3000$ \\
			MountainCar & $0.99$ & $[300,200]$ & $40$ & $1000$\\
			Acrobot & $0.995$ & $[128,64,64]$ & $500$ & $10000$ \\
			FrozenLake8x8 & $0.95$ & - & - & $80000$\\ \hline
		\end{tabular}
	\end{center}\caption{Details of hyper-parameters in Q-learning based algorithms experiments.} \label{tab:hp2}
\end{table}

For completeness, we also point out some weaknesses of the experiments we have done. Firstly, we are not optimizing over the parameter $r$ which is the radius of the spherical set associated with the uncertainty. We believe that performing a hyper-parameter search for the best $r$ will make the policy obtained by the RLSPI further robust. Secondly,  since we are focusing on the linear approximation architecture for developing the theoretical understanding of model-free robust RL, the experiments may not be immediately scalable to very high dimensional OpenAI Gym environments which typically require nonlinear approximation architecture. 

To end this section, we mention the software configurations used to generate these results: \emph{Python3.7 with OpenAI Gym \citep{brockman2016openai} and few basic libraries (non-exhaustive) like numpy, scipy, matplotlib}. Also, the hardware configurations used was \emph{macOS High Sierra Version 10.13.6, 16 GB LPDDR3, Intel Core i7.}

\end{document}